\documentclass{article}
\pdfoutput=1
\usepackage[preprint,nonatbib]{neurips_2025}
\usepackage[utf8]{inputenc} % allow utf-8 input
\usepackage[T1]{fontenc}    % use 8-bit T1 fonts
\usepackage{hyperref}       % hyperlinks
\usepackage{url}            % simple URL typesetting
\usepackage{booktabs}       % professional-quality tables
\usepackage{amsfonts}       % blackboard math symbols
\usepackage{nicefrac}       % compact symbols for 1/2, etc.
\usepackage{microtype}      % microtypography
\usepackage{xcolor}         % colors
\usepackage{algorithm}
\usepackage{algorithmic}
\usepackage{amsmath}
\usepackage[algo2e]{algorithm2e} 
\usepackage{amssymb}
\usepackage{amsthm}
\usepackage{framed}
\usepackage{xcolor}
\usepackage{tcolorbox}
\usepackage{adjustbox}

\tcbuselibrary{most}

\newtheorem{theorem}{Theorem}

\newtheorem{lemma}{Lemma}

\newtcolorbox{mybox}[2][]{colbacktitle=red!10!white, colback=blue!10!white,coltitle=red!70!black, title={#2},fonttitle=\bfseries,#1}

\title{HRP: High-Rank Preheating for Superior LoRA Initialization}

\author{%
\textbf{Yuzhu Chen}$^{1}$ \quad \textbf{Yingjie Wang}$^{2}$ \quad \textbf{Shi Fu}$^{1,2}$ \quad \textbf{Li Shen}$^{3}$ 
\\ \quad \textbf{Yongcheng Jing}$^{2}$ \quad \textbf{Xinmei Tian}$^1$ \quad \textbf{Dacheng Tao}$^2$ \\
$^1$University of Science and Technology of China \\
\quad $^2$Nanyang Technological University 
\quad $^3$Sun Yat-sen University
}

\begin{document}

\maketitle

\begin{abstract}
  This paper studies the crucial impact of initialization in Low-Rank Adaptation (LoRA). Through theoretical analysis, we demonstrate that the fine-tuned result of LoRA is highly sensitive to initialization, which is likely to lead suboptimal low-rank results. While this issue can be mitigated by adjusting the initial direction towards the main singular vectors of the target $\Delta W$, which is, however, typically unknown in real-world scenarios. To approximate this initial direction, we propose High-Rank Preheating (HRP), which first trains LoRA with a higher preheating rank for a few steps, then uses the main singular vectors of the derived $BA^\top$ as initialization for the main fine-tuning process. With only a modification in the initial direction, we prove that HRP makes LoRA achieve better fine-tuned results than random initialization in expectation, and the enhancement grows with the preheating rank. We validate our theoretical findings through extensive experiments in various models and tasks, where HRP significantly enhances LoRA's effectiveness and outperforms other initialization strategies and other LoRA variants.
\end{abstract}

\section{Introduction}

Recent advances in foundation models, especially large language models, have achieved remarkable success in a diverse range of applications~\cite{bommasani2021opportunities,touvron2023llama,achiam2023gpt,fu2024championing,fu2024towards}. Nevertheless, owing to their substantial scale, the conventional full-parameter fine-tuning (FPFT) approach, where all the model's parameters are updated for specialized tasks, has become progressively more formidable and inefficient. Parameter-efficient fine-tuning methods concentrate on integrating lightweight adapters, thus substantially diminishing computational and storage demands~\cite{hu2021lora,liu2022few,kumar2022fine}. This transition not only renders the fine-tuning process more tractable but also unlocks new prospects for deploying these potent models in resource-constrained settings. 

% PEFT to LoRA
A leading technique in this area is Low-Rank Adaptation (LoRA)~\cite{hu2021lora}, which introduces lightweight low-rank adapters to the pre-trained weight matrices. LoRA has been extensively applied and has manifested substantial achievements in tailoring large language models~\cite{hu2021lora,mao2025survey} and image generation models~\cite{filatov2023low,ji2024advlora} for a variety of downstream applications. Although LoRA presents significant computational benefits in practical scenarios, it still proves less effective than FPFT when computational cost is not a primary consideration~\cite{biderman2024lora}. 

% LoRA enhancement
To enhance LoRA's effectiveness, many variants have emerged, with initialization improvement being one line of approach. In classic LoRA, one adapter is initialized with a zero matrix and the other with a random matrix, which makes the fine-tuning process begin with a random direction. Methods like PiSSA~\cite{meng2024pissa} and LoRA-GA~\cite{wang2024lora} use Singular Value Decomposition (SVD) of pre-trained weights and gradients for initialization, proving the significance of LoRA initialization. However, these methods rely heavily on pre-trained models and lack theoretical guarantees for better performance, calling for more research to optimize LoRA initialization.

% Main of this paper
Delving into the optimization landscape of LoRA, this paper theoretically demonstrates that initialization plays a crucial role in achieving optimal performance. Nevertheless, random initialization leads to sub-optimal fine-tuning results, and a well-formed initialization can solve this problem. 
To approximate this well-formed initialization, we propose High-Rank Preheating (HRP), which further use a few steps of high-rank LoRA for enhancing initialization. 
% HRP not only inherits the convergence advantages related to high-rank LoRA but also keeps good generalization properties from low-rank LoRA by keeping the number of trainable parameters small.
% Experiments on GLUE show the benefits HRP brings. 
% 
Our contributions can be summarized as follows:

1. We theoretically demonstrate that initialization is important for LoRA to achieve optimal fine-tuned results. We analyze the gradient flow of classic LoRA and Asymmetric LoRA, one LoRA variant that only updates the zero-initialized matrices, and prove that: 1) with random initialization, Asymmetric LoRA can not achieve results better than random low-rank approximation in expectation, 2) classic LoRA can not achieve the best low-rank approximation from some initialization and has a similar dynamic with Asymmetric LoRA in the beginning stage, and 3) with wise initialization, both Asymmetric LoRA and classic LoRA converge exponentially to the best low-rank approximation. 

2. To approach the wise initialization suggested in theory, we propose High-Rank Preheating (HRP) that employs a few more steps of Asymmetric LoRA optimization before the main fine-tuning process for superior LoRA initialization. HRP sets the main singular vectors of the derived $BA^\top$ as LoRA initial directions, which approximate the main singular vectors of the target $\Delta W$. 
With only a modification in initialization, we prove that HRP makes LoRA achieve better converged results than random initialization in expectation.

3. To validate our theoretical findings and evaluate the effectiveness of HRP, we conducted experiments on natural language understanding (NLU) tasks and natural language generation (NLG) tasks across various models. In the results, HRP makes LoRA outperform most of its other variants and achieves comparable performance with full-parameter fine-tuning, demonstrating its effectiveness in real-world scenarios. With suitable hyperparameters, HRP requires relatively negligible more time and no more GPU memory.

\section{Related Work}
\label{sec:related}
% In this section, we provide an overview of the related work, including works on the analysis of initialization, LoRA variations, and theory results about matrix factorization. 

\paragraph{Role of initialization.}
Parameter initialization is one of the initial elements that largely account for the final model performance \cite{glorot2010understanding,mishkin2015all}.
Existing initialization methods are designed to control the norms of network parameters via Gaussian initialization \cite{he2015delving} or orthonormal matrix initialization \cite{saxe2013exact} with different variance patterns.
Currently, learning-based initialization methods are explored: \cite{dauphin2019metainit} proposes to optimize the curvature, \cite{zhu2021gradinit} suggests optimizing the loss reduction of the first stochastic step, while \cite{yang2022towards} optimizes the cosine similarity of sample-wise gradients. 

The initialization for LoRA is also a hot topic in previous research. 
\cite{hayou2024impact} study the difference between the left sketch and the right sketch from a stability perspective. \cite{buyukakyuz2024olora} leverages orthonormal matrix initialization through QR decomposition. \cite{meng2024pissa,wang2024lora} initialize adapters with the principal components of the weight matrices and their gradients in pre-trained models. \cite{li2024crucial} brings Nyström initialization to LoRA for better convergence. Compared to these works, our method does not require further knowledge or pre-trained weights. 

\paragraph{LoRA variations.}
Since introducing the original LoRA technique \cite{hu2021lora}, various efforts have been made to enhance LoRA further. \cite{zhang2023adalora} adaptively allocates the parameter budget among weight matrices. \cite{zhu2024asymmetry} freezes the random-initialized matrices for better generalization. \cite{xia2024chain,malinovsky2024randomized} suggest using a chain of LoRA for better expressive power. 
To further decrease the number of trainable parameters, \cite{balazy2024lora,ponkshe2024initialization} suggest injecting small matrices between LoRA blocks, and \cite{kopiczko2023vera,renduchintala2023tied,song2024sharelora} suggest sharing LoRA weights across different modules. Compared to these works, this paper focuses on enhancing LoRA from the perspective of initialization. 

\paragraph{Matrix factorization.}
We also note some works about matrix factorization here. Matrix factorization considers approximating a target matrix by two multiplied matrices \cite{chi2019nonconvex}. Theoretical works show that this training paradigm converges in both symmetric \cite{tarmoun2021understanding,min2021explicit} and asymmetric \cite{ye2021global,wind2023asymmetric} settings when adapters are initialized to small random matrices and the target has low rank. Compared to these works, this paper focuses more on the realistic setting for LoRA, where the target may have a high rank and initialization is not small.

\section{Theory: Why Initialization Matters}
\label{sec:theory}
In this section, we investigate the gradient flow of LoRA, demonstrating how its initialization influences the fine-tuned result. In the setting of matrix factorization, we show that LoRA with random initialization performs poorly in subsection~\ref{sec3.2} while LoRA with wise initialization performs well in subsection~\ref{theory-observation}.
% Specifically, we show: 1) the random initialized matrix is sometimes almost NOT optimized, yielding the dynamic of LoRA close to AsymLoRA~\cite{zhu2024asymmetry}, 2) with high probability of initialization, AsymLoRA can NOT converge to global minima in linear regression, and 3) at the beginning of LoRA, a better direction leads faster growth. 

\subsection{Framework}\label{framework}
LoRA~\cite{hu2021lora} adapts pre-trained models by updating weights through the product of two low-rank matrices scaled by a multiplier. Specifically, for a sub-module $W^{\operatorname{pre}}\in \mathbb{R}^{b\times a}$ in the pre-trained model, a $r$-rank LoRA takes $A\in\mathbb{R}^{a\times r}$ and $B\in\mathbb{R}^{b\times r}$ as adapters, and the weight adaption is given by 
$$W^{\operatorname{init}}\to W^{\operatorname{init}} + \frac{\alpha}{r} BA^\top,$$ 
where $\alpha$ denotes the scaling factor and the initial weights $W^{\operatorname{init}}=W^{\operatorname{pre}}$ are kept frozen during optimization. Only the parameters of $A$ and $B$ are updated through optimization algorithms. 

In this paper, we study the initialization of $A$ and $B$ through gradient flow, which approximates gradient descent. For loss function $\mathcal{L}$, the update rule of gradient descent is given by:
\begin{align*}
    &A_{t+1}=A_t-\frac{\eta_A\alpha}{r} \nabla_{W}\mathcal{L}\left(W+\frac{\alpha}{r}B_t^\top A_t\right)^\top B_t,~~~~
    % \\&
    B_{t+1}=B_t-\frac{\eta_B\alpha}{r} \nabla_{W}\mathcal{L}\left(W+\frac{\alpha}{r}B_t^\top A_t\right) A_t,
\end{align*}
where $\eta_A$ is the learning rate for optimizing $A$ while $\eta_B$ is for $B$. When the learning rate is sufficiently small, the update rule is a first-order approximation of the gradient flow:
\begin{align}
    \dot{A_t}=-\eta_A G_t^\top B_t,~~~~~~
    \dot{B_t}=-\eta_B G_t A_t,
    \label{abflow}
\end{align}
where $G_t$ denotes gradient $\frac{\alpha}{r} \nabla_W\mathcal{L}\left(W+\frac{\alpha}{r}B_t^\top A_t\right)$, notion $\dot{A_t}$ denotes $\frac{dA_t}{dt}$ and $\dot{B_t}$ is defined similarly. 

For the initialization of $A_t$ and $B_t$, \textbf{zero+random initialization schema} is widely used in LoRA, where one adapter is initialized to a zero matrix and another to a random matrix. Every element in the random-initialized matrix is independently and identically distributed from a Gaussian distribution $\mathcal{N}(0,\sigma^2)$. We call the $A_0=O_{a\times r}$ case left sketch initialization (LSI) and the $B_0=O_{b\times r}$ case right sketch initialization (RSI).~\cite{zhu2024asymmetry} suggest orthogonal initialization, which replaces the Gaussian matrix with its top $r$ singular vectors and demonstrates similar performance to Gaussian initialization. 

Updating $A$ with $\eta_A$ while updating $B$ with $\eta_B$ is suggested by~\cite{hayou2024lora}, which contains two special and widely-used subcases: 1) \textbf{classic LoRA} where $\eta_A=\eta_B=\eta$, and 2) \textbf{Asymmetric LoRA} where $\eta_A=0,\eta_B=\eta$ for RSI and $\eta_A=\eta,\eta_B=0$ for LSI. The classic LoRA is the original version proposed in~\cite{hu2021lora}. Asymmetric LoRA is a variant that freezes the random-initialized matrix, which reduces the number of trainable parameters and allows a higher rank. In this paper, we focus on how $A_0, B_0$ affect the performance of $A_t, B_t$ under these two updating strategies.

\subsection{Random initialization leads bad performance}\label{sec3.2}

As studied in~\cite{zeng2023expressive}, the expressive power of LoRA for fully connected neural networks and the Transformer architecture requires the capability of achieving the best low-rank approximation to some well-trained sub-modules during optimization (the analysis is optimization-free and thus also applicable for Asymmetric LoRA). Formally, for a target model $W^{\operatorname{target}}$, LoRA adapter $\frac{\alpha}{r}B_t A_t^\top$ is expected to achieve the best rank-$r$ approximation of $W^{\operatorname{target}}-W^{\operatorname{init}}$. 

This expected problem is expressed well in the problem of matrix factorization, where the loss function is the Frobenius norm toward the target. With no loss of generality, we consider $W^{\operatorname{init}}=O_{b\times a}$ while treating $M=W^{\operatorname{target}}-W^{\operatorname{init}}$ as target for $\frac{\alpha}{r}B_t A_t^\top$. Then, the loss function $\mathcal{L}$ and gradient $G_t$ can be expressed as
\begin{align}
    \label{matrix-factorization}
    \mathcal{L}_t=\frac{1}{2}\left\|\frac{\alpha}{r}B_t A_t^\top-M\right\|_F^2,~~\text{and}~~G_t=\frac{\alpha}{r} \left(\frac{\alpha}{r}B_t^\top A_t - M \right).
\end{align}
According to Eckart-Young Theorem~\cite{eckart1936approximation}, the global minima to $\frac{\alpha}{r}B^\top A$ is the best rank-$r$ approximation of $M$, with minimum loss as follows:
\begin{align*}
    \left(\frac{\alpha}{r}B A^\top\right)^*=\sum_{i=1}^{r}\sigma_i(M)u_i(M)v_i(M)^\top,~~~~~~
    \mathcal{L}^*=\frac{1}{2}\sum_{i=r+1}^{\min\{a,b\}}\sigma_i(M)^2
\end{align*}
where $\sigma_i(M)$ is the $i$-th large singular value of $M$ and $u_i(M),v_i(M)$ are the coresponding left and right singular vector. With no loss of generality, we assume the singular values of $M$ are different from each other. 
For expressive power in more complex tasks, LoRA is expected to at least achieve some points with loss similar to $\mathcal{L}^*$ in the setting of matrix factorization. Unfortunately, we show that fine-tuned results of Asymmetric LoRA and classic LoRA are likely to have a significantly higher loss. 

\paragraph{Asymmetric LoRA struggles to perform well.}
% We begin by examining the optimization landscape of Asymmetric LoRA~\cite{zhu2024asymmetry}, which distinguishes itself from the classic LoRA by keeping the random-initialized matrix from being updated. Compared to the classic LoRA with an equivalent rank, Asymmetric LoRA results in a reduction of trainable parameters, thereby enhancing generalization capabilities. Its success in experiments demonstrates that the frozen parameters do not substantially compromise the convergence of classic LoRA. 
%
We begin by examining the optimization landscape of Asymmetric LoRA~\cite{zhu2024asymmetry}. Due to the fixed random-initialized LoRA adapter, the derived $\frac{\alpha}{r}B_t A_t$ is limited by the direction of the frozen adapter, which is always not fit to target $M$. Theoretically, we show that when the frozen adapter is randomly settled, the expectation loss of Asymmetric LoRA is lower bounded by the loss of random rank-$r$ approximation, which is sometimes much higher than $\mathcal{L}^*$. 
Specifically, we have the following theorem. 
\begin{theorem}\label{ms-loss}
    Consider Asymmetric LoRA under objective~\ref{matrix-factorization} in gradient flow~\ref{abflow} with the frozen adapter from Gaussian initialization or orthogonal initialization. For LSI and RSI, we have for all $t>0$:
    \begin{align*}
        \mathbb{E}_{LSI}\left[\mathcal{L}_t\right]\geq \frac{b-r}{2b}\sum_{i=1}^{\min\{a,b\}}\sigma_i(M)^2,~~~~~~
        \mathbb{E}_{RSI}\left[\mathcal{L}_t\right]\geq \frac{a-r}{2a}\sum_{i=1}^{\min\{a,b\}}\sigma_i(M)^2,
    \end{align*}
    where $\mathbb{E}$ represents the expectation with respect to randomness in initialization. The inequality becomes an equality when $t\to\infty$. 
\end{theorem}
Proof is included in Appendix~\ref{ms-loss-proof}. Theorem~\ref{ms-loss} shows that from the perspective of expectation loss, Asymmetric LoRA can not achieve results better than random rank-$r$ approximation with a random frozen adapter. Compared with the best rank-$r$ approximation, the lower bound is much higher than $\mathcal{L}^*$ when the singular values of target $M$ diverge a lot, e.g. when $M$ has low effective rank. As observed by~\cite{wang2023cuttlefish}, the rank-dimension ratio of well-trained neural networks is usually relatively small, indicating a low effective rank structure for both $W^{\operatorname{pre}}$ and $W^{\operatorname{target}}$, thus also for $M=W^{\operatorname{target}}-W^{\operatorname{init}}$. Therefore, under zero+random initialization schema, the performance of Asymmetric LoRA is always much less than is expected. 

Besides, we acknowledge that the lower bound of Theorem~\ref{ms-loss} decreases linearly with the LoRA rank $r$. This decreasing rate is more pronounced than that of $\mathcal{L}$, suggesting that the performance gap to the best rank-$r$ approximation diminishes as $r$ increases. Therefore, employing a higher rank in Asymmetric LoRA is a way to enhance performance. Specifically, when using Asymmetric LoRA with full rank ($r=\min{a,b}$), it can converge to global minima with probability one. However, totally using higher rank for main fine-tuning process is equivalent to FPFT, which needs more computation and GPU memory usage in practical applications. 

% \begin{theorem}
%     \label{ms-pr}
%     Under the conditions of Theorem~\ref{ms-loss}, with the additional assumption that $r < \min{a,b}$, we have:
%     \begin{align*}
%         \operatorname{Pr}\left[\mathcal{L}(\lim_{t\to\infty}X_t)=\mathcal{L}^*\right]=0,
%     \end{align*}
%     where $\operatorname{Pr}$ represents the probability with respect to randomness in initialization. 
% \end{theorem}
% Proof is included in Appendix~\ref{ms-pr-proof}. Aside from the high expected loss, Theorem~\ref{ms-pr} further reveals that Asymmetric LoRA has zero probability of converging to the optimal rank-$r$ approximation of $M$ in matrix factorization problems. This indicates that the high expected loss is not merely due to occasional poor outcomes, but rather stems from the inherent nature of zero+random initialization Asymmetric LoRA to converge to arbitrary low-rank solutions. 

Compared with other theoretical works in matrix factorization~\cite{tarmoun2021understanding, min2021explicit, ye2021global, wind2023asymmetric}, which demonstrates that the gradient flow of matrix factorization converges to global minima, our approach suggests a different conclusion because of difference in the following ways: 1) we allow the target matrix $M$ to be of high rank, which may not be fully approximate-able by a low-rank matrix $\frac{\alpha}{r}BA^\top$; 2) we consider the case where adapter matrices are initialized as zero and random, respectively, whereas the aforementioned works assume both matrices are initialized with small random values; and 3) our objective is Asymmetric LoRA, where only one matrix is optimized, which is different from their approaches.

\paragraph{Classic LoRA has similar properties.}
\label{lora-approx-assymetric}
Different from Asymmetric LoRA, classic LoRA lets the random-initialized to be updated, thus has more trainable parameters and is more expressive. Though theory papers in matrix factorization~\cite{tarmoun2021understanding, min2021explicit, ye2021global, wind2023asymmetric} suggest that optimizing both $A$ and $B$ results in converging to $M$ with high probability when both $A$ and $B$ are initialized from small random matrices, we demonstrate that under zero+random initialization and fine-tuning schema, classic LoRA also struggles to perform well. First, we show that for some special initialization, classic LoRA can not achieve to the best low-rank result in matrix factorization. 
\begin{theorem}\label{lora-bad}
    For classic LoRA under objective~\ref{matrix-factorization} in gradient flow~\ref{abflow}, if there exists $i\leq r$ making the initial $A_0$ and $B_0$ satisfy $A_0^\top v_i(M)=O_a$ and $B_0^\top u_i(M)=O_b$, we have for all $t\geq 0$:
    \begin{align*}
        \left(\frac{\alpha}{r}B_t A_t^\top\right)v_i=O_b,~~~\text{and}~~~u_i^\top\left(\frac{\alpha}{r}B_t A_t^\top\right)=O_a^\top,
    \end{align*}
    resulting in for all $t\geq 0$:
    \begin{align*}
        \mathcal{L}_t-\mathcal{L}^*\geq \frac{1}{2}[\sigma_i(M)^2-\sigma_{r+1}(M)^2]>0,
    \end{align*}
    where $\sigma_i(M)$ is the $i$-th large singular value of $M$ and $u_i(M),v_i(M)$ are the corresponding left and right singular vector. 
\end{theorem}
Proof is included in Appendix~\ref{lora-bad-proof}. Theorem~\ref{lora-bad} shows that when the initialization of LoRA lies on a subspace orthogonal to some $u_i(M)v_i(M)^\top$ from the target, the orthogonal property will always hold in $X_t$, leading to a failure in converging to the best low-rank result. 
In zero+random initialization, the orthogonal property to the zero-initialized adapter is always satisfied ($A_0^\top v_i(M)=O_a$ for LSI and $B_0^\top u_i(M)=O_b$ for RSI). This means that once the random-initialized matrix is orthogonal to some main singular vectors of the target, classic LoRA can never achieve the best rank-$r$ approximation. 
The target is always unknown in the initializing stage, thus these bad directions can not be easily avoided from random initialization. 

Aside from these specific initialization directions, we also observe that classic LoRA has a similar optimization dynamic with Asymmetric LoRA when the fine-tuning process is not sufficiently comprehensive, thus sharing similar properties including sub-optimal performance. This similar dynamic is empirically supported by experiments~\cite{zhu2024asymmetry} where Asymmetric LoRA obtains similar results with classic LoRA in real-world scenarios. To theoretically demonstrate this similarity, we have the following theorem. 
\begin{theorem}\label{lora-similar}
    Consider the gradient flow of classic LoRA $B_t,A_t$ and that of Asymmetric LoRA $\tilde{B_t},\tilde{A_t}$ in the same zero+random initialization. Assume computed gradient $G_t$ is bounded by $R$ and is Lipschitz in the Frobenius norm, then the difference between Asymmetric LoRA and classic LoRA is upper bounded by
    \begin{align}
        \left\|\frac{\alpha}{r}B_t A_t^\top - \frac{\alpha}{r}\tilde{B_t} \tilde{A_t}^\top\right\|_F=O\left(\eta^4R^3rt^4\right),
    \end{align}
    for small $t$ when they are under a same gradient calculator. 
\end{theorem}
Proof is included in Appendix~\ref{lora-similar-proof}. 
In fine-tuning tasks, pre-trained models have already acquired powerful capabilities from training on other tasks, enabling their effectiveness in real-world applications. Therefore, the fine-tuning stage is expected to make only modest modifications to the model parameters, making the assumption of bounded gradient norm reasonable. Additionally, assuming the gradient calculator is Lipschitz in the Frobenius norm is justified because the backward propagation process of neural networks exhibits this property in bounded domains. Besides, this Lipschitz assumption is widely adopted in other theoretical works~\cite{patel2022global}.

Theorem~\ref{lora-similar} tells that in many fine-tuning tasks, classic LoRA exhibits dynamics similar to Asymmetric LoRA with the same zero+random initialization, particularly during the early stages of training (small $t$). As addressed above, Asymmetric LoRA tends to achieve random low-rank results rather than the best low-rank approximation. Given this dynamic similarity, classic LoRA inherits the same limitation when the fine-tuning process is not sufficiently comprehensive.

\subsection{Wise initialization leads good convergence}\label{theory-observation}

Then, we illustrate that the above limitation of Asymmetric LoRA and classic LoRA is mainly attributed to initialization. Specifically, we show that by merely altering the random initialization to a well-formed direction, both Asymmetric LoRA and classic LoRA converge to the global minima exponentially in the matrix factorization problem. 
\begin{theorem}
    \label{asym-wise}
    Consider Asymmetric LoRA under objective~\ref{matrix-factorization} in gradient flow~\ref{abflow}. For RSI with $A_0=\sum_{i=1}^{r}v_i(M)e_{i,r},B_0=O_{b\times r}$ or LSI with $A_0=O_{a\times r},B_0=u_i(M)e_{i,r}$, we have
    \begin{align*}
        \mathcal{L}_t-\mathcal{L}^*=O\left(\exp\left\{-2\alpha\eta t/r\right\}\right),
    \end{align*}
    where $v_i(M)$ and $u_i(M)$ are the $i$-th right and left singular vector of $M$ and are defined in Theorem~\ref{lora-bad}. $e_{i,r}$ is the $1\times r$ vector with the i-th element being 1 and all other elements being 0. 
\end{theorem}
Proof is included in Appendix~\ref{asym-wise-proof}. 
Theorem~\ref{asym-wise} shows that when initialized with the main singular vectors of the target, Asymmetric LoRA has exponential convergence to the best rank-$r$ approximation of $M$. Compared with Theorem~\ref{ms-loss}, the only difference in the setting lies in the initialization. Meanwhile, the conclusion shifts from `can not perform better than random rank-$r$ approximation' to `converges to the best rank-$r$ approximation exponentially'. Therefore, the performance limitation of Asymmetric LoRA is attributed to random initialization. Furthermore, when the initialization is appropriately configured, optimizing a single LoRA adapter while freezing another one is sufficient to ensure performance in matrix factorization.

\begin{theorem}
    \label{lora-wise}
    Consider classic LoRA under objective~\ref{matrix-factorization} in gradient flow~\ref{abflow}. For initialization same as Theorem \ref{asym-wise}, we have
    \begin{align*}
        \mathcal{L}_t-\mathcal{L}^*=O\left(\exp\left[-\frac{2\alpha\eta t}{r}-\frac{\alpha\eta t}{r}\left(-1+\sqrt{1+\frac{2r^2}{\alpha^2} \sigma_r(M)^2\left[1-e^{-\frac{\alpha\eta t}{2r}}\right]^2 }\right)\right]\right),
    \end{align*}
    where $\sigma_i(M)$ is the $i$-th singular value of $M$ and is defined in Theorem~\ref{lora-bad}.
\end{theorem}
Proof is included in Appendix~\ref{lora-wise-proof}. 
Theorem~\ref{lora-wise} demonstrates that the initialization suggested in~\ref{asym-wise} also makes classic LoRA converge to the best rank-$r$ result exponentially, thereby avoiding the unfavorable scenario described in Theorem~\ref{lora-bad}. Similar to analysis in Asymmetric LoRA, the performance limitation of classic LoRA is also attributed to random initialization. It is also noteworthy that with the same well-formed initialization, the convergence rate of classic LoRA has an extra positive term $\frac{\alpha\eta t}{r}\left(-1+\sqrt{1+\frac{2r^2}{\alpha^2} \sigma_r(M)^2\left[1-e^{-\frac{\alpha\eta t}{2r}}\right]^2 }\right)$ from the convergence rate of Asymmetric LoRA, indicating a better fine-tuned result in the perspective of convergence. 

In practice with pre-trained model modules $W^{\operatorname{pre}}$ and its well fine-tuned version $W^{\operatorname{target}}$, LoRA is expected to make each $\frac{\alpha}{r}B_iA_i^\top$ a best rank-$r$ approximation of $W_i^{\operatorname{target}}-W_i^{\operatorname{pre}}$. However, calculating the SVD decomposition of $W_i^{\operatorname{target}}-W_i^{\operatorname{pre}}$ is always impossible because $W_i^{\operatorname{target}}$ is unknown. 
Previous work proposes some initialization methods for LoRA, such as PiSSA~\cite{meng2024pissa} and LoRA-GA~\cite{wang2024lora}. PiSSA~\cite{meng2024pissa} considers using main singular vectors of $W^{\operatorname{pre}}$ as initialization while LoRA-GA~\cite{wang2024lora} suggests using main singular vectors of $\nabla_{W^{\operatorname{pre}}}\mathcal{L}(W^{\operatorname{pre}})$. Their methods rely on the information from a single weight point $W^{\operatorname{pre}}$, which could be easily affected by noise from the pre-trained model and gradient noise. 
% However, their methods rely on the similarity between $W_i^{\operatorname{target}}-W_i^{\operatorname{pre}}$ and $W_i^{\operatorname{pre}}$. 

\section{Method: High-Rank Preheating}\label{sec:hrp}

\begin{figure}
    \centering
    \includegraphics[width=.45\textwidth, bb=0 0 461.07 345.41]{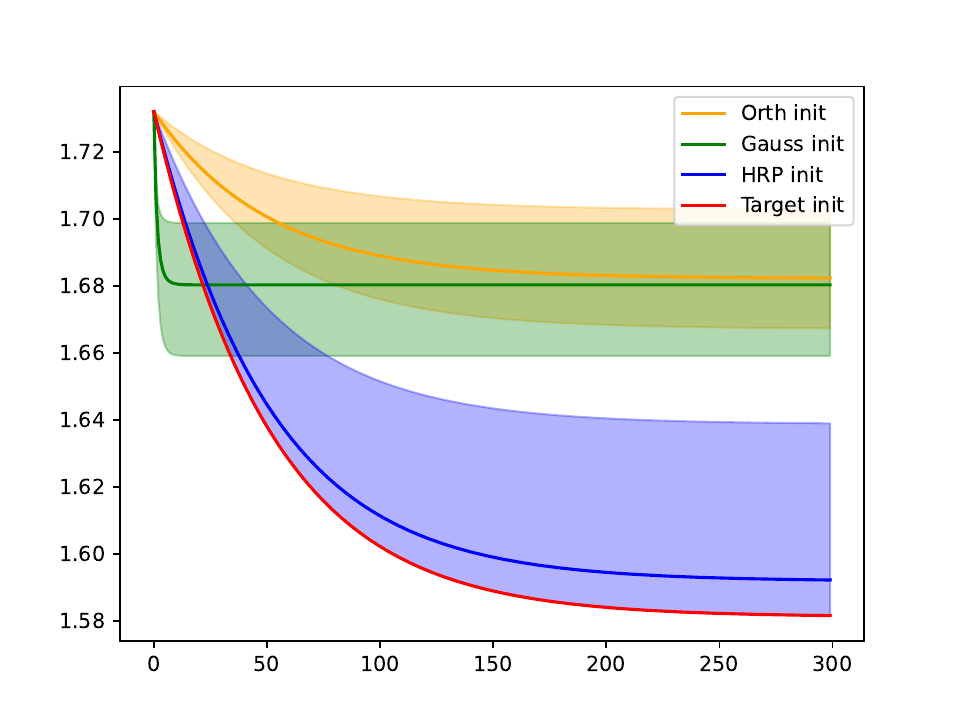}
    \includegraphics[width=.45\textwidth, bb=0 0 461.07 345.41]{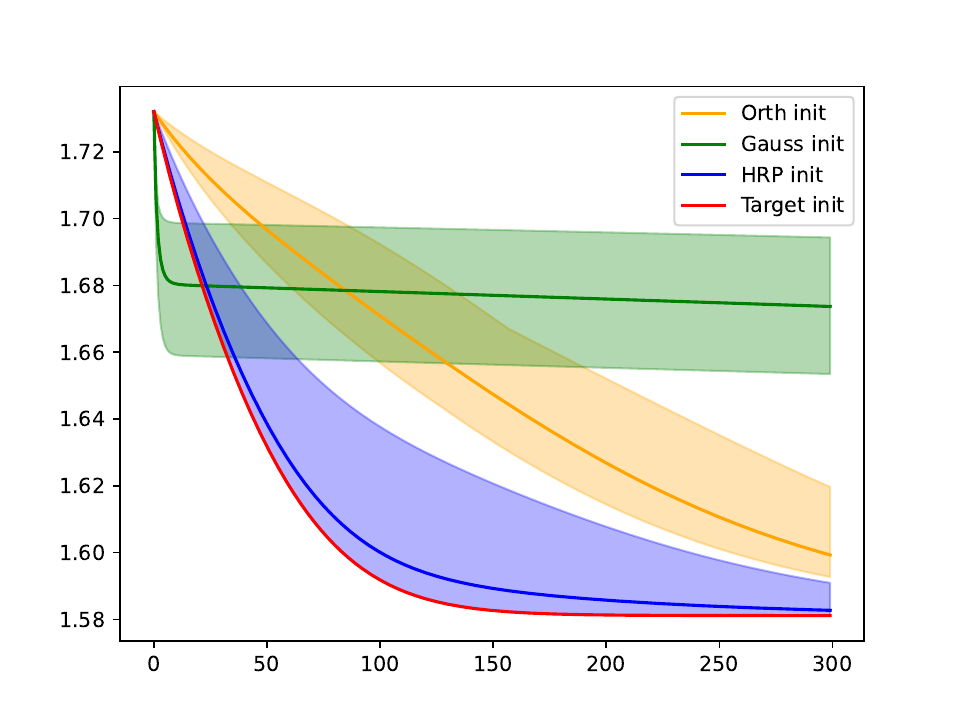}
    \caption{\label{fig-mf}Loss curves for matrix factorization targeting $M=\operatorname{diag}(I_{12},O_{20\times 20})$ with $r=2$. Left: classic LoRA in Gaussian initialization, orthogonal initialization, HRP derived initialization with $\operatorname{hrp\_rank}=6$, and target initialization (suggested in Theorem~\ref{asym-wise}). Right: Asymmetric LoRA in the same initialization strategies. }
\end{figure}

% In this section, we introduce High-Rank Preheating (HRP), our proposed LoRA initialization algorithm for addressing the weaknesses identified in Section~\ref{sec3.2}. 
% The target of HRP is approaching the wise initialization suggested in Theorem~\ref{asym-wise}. 

Though Asymmetric LoRA can not achieve the best rank-$r$ approximation of the target directly, it does approximate the target in some random direction, thus bringing out some information about the target implicitly. HRP leverages this implicit information to approximate the wise initialization suggested in Theorem~\ref{asym-wise}. 
Specifically, when using the RSI strategy, HRP can be formulated as Algorithm~\ref{hrp-alg}. 

\begin{algorithm}\label{hrp-alg}
\caption{High-Rank Preheating for RSI}
\KwIn {HRP parameters $\operatorname{hrp\_rank},\operatorname{hrp\_step}$ and $\operatorname{hrp\_bs}$, rank $r$ for LoRA, pre-trained model $W$}
\begin{algorithmic}[1]
\STATE Select random orthogonal matrix $\hat{U}\in\mathbb{R}^{b\times b}$ with $i$-th column $\hat{u_i}$
\STATE High-rank preheating in batch size $\operatorname{hrp\_bs}$: $$\hat{A}, \hat{B}_0\leftarrow \operatorname{hrp\_step}\text{ steps AsymLoRA with LSI: }\hat{A}_0=O_{a\times r}, \hat{B}_0=\sum_{i=1}^{\operatorname{hrp\_rank}}\hat{u_i}e_{i,\operatorname{hrp\_rank}}^\top$$
\STATE Calculate first $r$ right singular vectors $v_i(\hat{B}\hat{A}_0^\top)$ for $i=1,\cdots,r$
\STATE Fine-tuning $$A, B\leftarrow \text{ LoRA fine-tuning with RSI: }A_0=\sum_{i=1}^rv_i(\hat{B}_0\hat{A}^\top)e_{i,r}^\top, \hat{B}_0=O_{b\times r}$$
\end{algorithmic}
\KwOut {fine-tuned model $W+\frac{\alpha}{r}BA^\top$}
\end{algorithm}

To illustrate why HRP approximates the well-formed initialization, we take LoRA for RSI in matrix factorization as an example, where the well-formed initialization requires the main $r$ right singular vectors to initialize $A$. One step HRP updates $\hat{A_0}$ from $O_{a\times r}$ to 
\begin{align*}
    \hat{A}_1 = \hat{A}_0-\frac{\eta_A\alpha}{r} \nabla_{W}\mathcal{L}\left(W+\frac{\alpha}{r}\hat{B}_0^\top \hat{A}_0\right)^\top \hat{B}_0 = \frac{\eta\alpha}{r}M^\top\hat{B}_0,
\end{align*}
while the derived $\Delta W$ updates from $\frac{\alpha}{r}\hat{B}_0\hat{A}_0^\top=O_{b\times a}$ to 
\begin{align*}
    \frac{\alpha}{r}\hat{B}_1\hat{A}_1^\top = \frac{\eta\alpha^2}{r^2}\hat{B}_0\hat{B}_0^\top M = \frac{\eta\alpha^2}{r^2}\hat{U}\begin{pmatrix}
        I_{\operatorname{hrp\_rank}}\\&O_{(b-\operatorname{hrp\_rank})\times (b-\operatorname{hrp\_rank})}
    \end{pmatrix}\hat{U}^\top M.
\end{align*}
Due to the random selection of $\hat{U}$, $\frac{\alpha}{r}\hat{B}_1\hat{A}_1^\top$ tends to have similar main singular vectors with $M$. With a higher $\operatorname{hrp\_rank}$, more effective directions from $\hat{U}$ are used for approaching $M$, thus approaching $M$ better. Specifically, we prove that HRP achieves the initialization that leads to better converged results in matrix factorization. 
\begin{theorem}\label{hrp-theory}
    For any $\operatorname{hrp\_step}>r$ in objective~\ref{matrix-factorization} with gradient flow, HRP initialization makes Asymmetric LoRA converge to results with expectation loss
    \begin{align*}
        \mathbb{E}_{\operatorname{HRP}}\mathcal{L}_{\infty}\leq\sum_{i=r+1}^{\operatorname{hrp\_rank}}\sigma_i(M)^2+\frac{a-\operatorname{hrp\_rank}}{2a}\sum_{i=1}^{\max(a,b)}\sigma_i(M)^2.
    \end{align*}
\end{theorem}
Proof is included in Appendix~\ref{hrp-theory-proof}. 
Compared with the lower bound of Theorem~\ref{ms-loss}, the upper bound of expectation loss in Theorem~\ref{hrp-theory} modifies $\operatorname{hrp\_rank}-r$ terms of the averaged singular value with $r+1$-th to $\operatorname{hrp\_rank}$-th singular value. This means that when the target $W^{\operatorname{target}}-W^{\operatorname{pre}}$ has a small effective rank, which is observed by \cite{wang2021pufferfish,wang2023cuttlefish} in practice, HRP improves the converged result a lot. Besides, we note that the upper bound is decreasing with $\operatorname{hrp\_rank}$ growing, which is why we use a higher rank for preheating. 

% Theorem~\ref{hrp-theory} tells that when target $M$ is able to be perfectly represented by LoRA adapters, HRP gets the initialization that enables Asymmetric LoRA converge $M$ exponentially. 
% We note that assuming $W^{\operatorname{target}}-W^{\operatorname{pre}}$ to have a low rank is reasonable in practice. \cite{wang2021pufferfish,wang2023cuttlefish} observe a stabilizing effect in the stable ranks of neural network layers during training, indicating both $W^{\operatorname{pre}}$ and $W^{\operatorname{target}}$ having small stable rank. 

Aside from theoretical analysis, we conduct numerical experiments and show results in Figure~\ref{fig-mf}. As shown, HRP indeed achieves better performance than random initialization, while approaching the performance of target initialization. 

In real-world scenarios, HRP requires some modifications for fit more complex problems. Firstly, to mitigate the problem of gradient noise, HRP needs more than one preheating step in real-world experiments. Besides, HRP uses high rank Asymmetric LoRA, which requires more GPU memory. To mitigate this problem, we select a smaller batch size in the HRP process to save GPU memory cost. Furthermore, calculating SVD decomposition is sometimes of high cost when dealing with wide neural networks. To accelerate HRP, we use the random SVD algorithm~\cite{voronin2015rsvdpack} in step 3 of Algorithm~\ref{hrp-alg}. 

Compared with totally fine-tuning with high-rank LoRA, HRP maintains the same number of trainable parameters as low-rank LoRA in the main fine-tuning process, thereby preserving the better generalization guarantee against totally high-rank LoRA (refer to Lemma 4.5 in~\cite{zhu2024asymmetry} where generalization error is upper bounded by $O(\sqrt{r})$). Furthermore, with the decreased batch size in HRP, it costs much more time for totally fine-tuning high-rank LoRA with the same epochs as HRP-initialized low-rank LoRA. 

We also compare HRP with other LoRA initialization - PiSSA~\cite{meng2024pissa} and LoRA-GA~\cite{wang2024lora}, where neither $A_0$ nor $B_0$ is initialized as a zero matrix and $W^{\operatorname{init}}$ is settled to $W^{\operatorname{pre}}-\frac{\alpha}{r}B_0A_0^\top$ to ensure that $W+\Delta W$ is fine-tuned from $W^{\operatorname{pre}}$. This has implications for checkpoint storage, where sometimes only the trained LoRA weights need to be shared. While with these methods, either both $B^{\operatorname{init}}, A^{\operatorname{init}}$ and $B^{\operatorname{tuned}}, A^{\operatorname{tuned}}$ need to be stored, or the weights must be merged, both of which require more memory. Compared with their methods, HRP keeps the properties $B_0A_0^\top=O_{b\times a}$ and $W^{\operatorname{init}}=W^{\operatorname{pre}}$ from LoRA, thus does not require more checkpoint storage.

\begin{table*}
    \centering
    \caption{\label{glue}Results with T5-base on tasks from a subset of the GLUE benchmark. \textcolor{blue}{Blue} marked denotes the best result across different initialization, while \textcolor{red}{red} marked denotes the best result over all baselines.}
\begin{tabular}{lllllll}
\toprule
 & CoLA & MRPC & QNLI & RTE & STS-B & Avg. \\\hline
 AdaLoRA & $54.78_{ \pm 1.83 }$ & $88.64_{ \pm 0.31 }$ & $92.68_{ \pm 0.07 }$ & $74.13_{ \pm 1.62 }$ & $88.50_{ \pm 0.31 }$ & $79.75_{ \pm 0.40 }$ \\
 rsLoRA & $54.91_{ \pm 1.06 }$ & $89.05_{ \pm 0.31 }$ & $92.07_{ \pm 0.13 }$ & $73.29_{ \pm 0.88 }$ & $89.17_{ \pm 0.16 }$ & $79.70_{ \pm 0.26 }$ \\
 DoRA & $54.61_{ \pm 1.87 }$ & $88.40_{ \pm 1.10 }$ & $91.96_{ \pm 0.07 }$ & $73.85_{ \pm 0.17 }$ & $88.93_{ \pm 0.36 }$ & $79.55_{ \pm 0.53 }$ \\
 FPFT & $55.85_{ \pm 0.64 }$ & \textcolor{red}{$89.46_{ \pm 1.06 }$} & \textcolor{red}{$92.50_{ \pm 0.14 }$} & $74.01_{ \pm 1.93 }$ & \textcolor{red}{$89.34_{ \pm 0.57 }$} & \textcolor{red}{$80.23_{ \pm 0.42 }$} \\\hline
 LoRA & $54.93_{ \pm 1.59 }$ & $88.32_{ \pm 0.81 }$ & $91.96_{ \pm 0.19 }$ & $73.65_{ \pm 2.06 }$ & $88.83_{ \pm 0.39 }$ & $79.54_{ \pm 0.78 }$ \\
 LoRA (orth) & $54.84_{ \pm 0.27 }$ & $88.48_{ \pm 0.53 }$ & $91.95_{ \pm 0.16 }$ & $73.89_{ \pm 2.74 }$ & $88.87_{ \pm 0.11 }$ & $79.60_{ \pm 0.57 }$ \\
 PiSSA & $55.30_{ \pm 0.53 }$ & $88.73_{ \pm 1.00 }$ & $91.74_{ \pm 0.27 }$ & $73.04_{ \pm 1.62 }$ & $88.54_{ \pm 0.21 }$ & $79.47_{ \pm 0.64 }$ \\
 LoRA-GA & $53.04_{ \pm 1.22 }$ & $88.64_{ \pm 0.81 }$ & $92.13_{ \pm 0.13 }$ & $72.08_{ \pm 2.01 }$ & $88.84_{ \pm 0.39 }$ & $78.95_{ \pm 0.53 }$ \\
 HRP (ours) & \textcolor{red}{$56.18_{ \pm 0.09 }$} & \textcolor{blue}{$88.89_{ \pm 0.61 }$} & \textcolor{blue}{$92.23_{ \pm 0.08 }$} & \textcolor{red}{$74.25_{ \pm 1.90 }$} & \textcolor{blue}{$89.04_{ \pm 0.16 }$} & \textcolor{blue}{$80.12_{ \pm 0.52 }$} \\
\bottomrule
\end{tabular}

\end{table*}

\section{Experiments}\label{exp}

In this section, we validate the effectiveness of HRP through experiments. 
We compare HRP with several baselines to demonstrate its effectiveness: 1) Full-Parameter Fine-Tuning (FPFT): updates all model parameters from pre-trained weights; 2) other LoRA variants, including: 2.1) DoRA \cite{liu2024dora}: with additional learnable magnitudes, 2.2) rsLoRA \cite{kalajdzievski2023rank}: with a scaling factor for stability, and 2.3) AdaLoRA \cite{zhang2023adalora}: with dynamically adjusted rank allocation; and 3) classic LoRA with different initialization methods: 3.1) kaiming normal initialization, 3.2) orthogonal initialization \cite{zhu2024asymmetry}, 3.3) PiSSA \cite{meng2024pissa}: first $r$ main singular vectors of $W^{\operatorname{pre}}$, and 3.4) LoRA-GA \cite{wang2024lora}: first $2r$ main singular vectors of gradient approximation. 

\subsection{Experiments on NLU tasks}
We first evaluate the performance of HRP and other LoRA variants under the natural language understanding (NLU) tasks. 
Specifically, we fine-tune the T5-base model \cite{2020t5} on a subset of GLUE \cite{wang2018glue} benchmark, including CoLA, MRPC, QNLI, RTE, and STS-B. 
%  in Table~\ref{glue}. 
For all variants of LoRA, we inject LoRA blocks to all query and value sub-modules with low rank $r=4$ and $\alpha=r$. For HRP, we set $\operatorname{hrp\_rank}=128$, $\operatorname{hrp\_bs}=16$, and $\operatorname{hrp\_step}=100$ with the same training dataset as the main fine-tuning process. 
We present more implementation details in Appendix~\ref{apd-exp-detail}. 

Results are shown in Table~\ref{glue} where performance is evaluated on the Matthews correlation coefficient for CoLA, Pearson correlation coefficient for STS-B, and accuracy for the remaining tasks. Each experiment is conducted for 3 different random seeds (fixed seeds across different methods), and both the average and standard deviation are reported. 
As demonstrated, HRP outperforms all other initialization strategies and most of the other variants while achieving results comparable to FPFT. 
% To provide deeper insights into the training dynamics, we present the loss curves in Appendix~\ref{lc-nlg}. 
Besides, we report the time and GPU cost in Table~\ref{cost}. As shown, HRP costs negligible time compared to the subsequent fine-tuning process, while it needs no more GPU memory. 
To study how HRP hyperparameters affect the fine-tuned result, we include ablation studies in Appendix~\ref{ablation}. 

\subsection{Experiments on NLG tasks}

\begin{table*}
    \centering
    \caption{\label{llm}Results with LLMs on math reasoning tasks. \textcolor{blue}{Blue} marked denotes the best result across different initialization, while \textcolor{red}{red} marked denotes the best result over all baselines. }
\begin{tabular}{l|cccc|cccc}
\hline
 & \multicolumn{4}{c}{GSM8K} & \multicolumn{4}{|c}{MATH} \\\hline
 & ~Llama~ & ~Qwen~ & Gemma & ~Avg.~ & ~Llama~ & ~Qwen~ & Gemma & ~Avg.~ \\\hline
 AdaLoRA & 44.73 & 69.98 & 54.44 & 56.38 & \textcolor{red}{15.70} & 25.72 & 15.68 & 19.03 \\
 rsLoRA & 45.72 & 69.14 & 53.75 & 56.20 & 14.90 & 26.62 & 14.86 & 18.79 \\
 DoRA & 48.37 & 70.13 & 56.18 & 58.23 & 15.08 & 26.84 & 16.26 & 19.39 \\
 FPFT & 46.93 & \textcolor{red}{71.49} & 49.66 & 56.03 & 12.78 & 26.84 & 10.76 & 16.79 \\\hline
LoRA & 44.88 & 69.22 & 51.93 & 55.34 & 13.64 & 23.48 & 13.42 & 16.85 \\
LoRA (orth) & 41.77 & 66.11 & 41.39 & 49.76 & 11.12 & 22.16 & 10.30 & 14.53 \\
PiSSA & 47.08 & \textcolor{blue}{70.89} & 56.10 & 58.02 & 15.10 & 26.52 & \textcolor{red}{16.42} & 19.35 \\
LoRA-GA & 48.29 & 69.98 & 56.48 & 58.25 & 14.54 & 26.88 & 15.54 & 18.99 \\
HRP (ours) & \textcolor{red}{48.67} & 70.66 & \textcolor{red}{58.38} & \textcolor{red}{59.24} & \textcolor{blue}{15.62} & \textcolor{red}{27.24} & 15.98 & \textcolor{red}{19.61} \\
% HRP (ours) & \textcolor{red}{47.69} & 70.89 & 21.38 & 46.65 & \textcolor{red}{18.90} & \textcolor{blue}{32.78} & 3.30 & 18.33 \\
\bottomrule
\end{tabular}
\end{table*}

\begin{table*}
    \centering
    \caption{\label{cost}Time and GPU memory cost of LoRA-GA, HRP, and the following fine-tuning process. We report the computational cost with QNLI for T5-base while MetaMathQA for other models.}
\begin{tabular}{l|rrrr|rrrr}
    \hline
    & \multicolumn{4}{c}{Time (s)} & \multicolumn{4}{|c}{GPU Memory (MB)} \\\hline
    & T5-base & Llama & Qwen & Gemma & T5-base & Llama & Qwen & Gemma \\\hline
    % LoRA-GA & - & 103s & 67s & 249s & 18.00 & 30.76 &  &  \\
    HRP & 15 & 88 & 95 & 117 & 5106 & 24414 & 35080 & 53560 \\
    Main tuning & 3194 & 1721 & 2769 & 3626 & 8040 & 38702 & 56368 & 93704 \\
\bottomrule
\end{tabular}
\end{table*}

Additionally, we conduct experiments under the natural language generation (NLG) tasks. 
We fine-tune the Llama3.2-1B model~\cite{grattafiori2024llama}, Qwen3-1.7B model~\cite{qwen3}, and Gemma2-2B model~\cite{gemma_2024} by AdamW \cite{loshchilov2019decoupledweightdecayregularization} with batch size 16 and a cosine learning rate schedule on a 50K subset of the MetaMathQA dataset \cite{yu2023metamath} for 1 epoch. Then, we evaluate the fine-tuned models on the test set of GSM8K \cite{cobbe2021gsm8k} and MATH \cite{hendrycksmath2021}. 
During fine-tuning, we fix the same random seed in all baselines and set the learning rate to $4\times 10^{-4}$ for all LoRA variants, while $5\times 10^{-5}$ for FPFT. 

For all variants of LoRA, we inject LoRA blocks for all linear sub-modules except $\operatorname{lm\_head}$, with rank $r=8$ and $\alpha=r$ in the main fine-tuning process. For HRP, we set the preheating rank $\operatorname{hrp\_rank}=256$ for $\operatorname{hrp\_bs}=8$ with $\operatorname{hrp\_step}=200$ steps in the same training dataset with the AdamW optimizer under a constant learning rate. Our model is fine-tuned using a standard supervised learning fine-tuning schema for language modeling, where the loss for the input prompt is set to zero. We present more implementation details in Appendix~\ref{apd-exp-detail}. 

We report the evaluated results in Table~\ref{llm}, which demonstrate the effectiveness of HRP across most large language models. 
However, we also observe that HRP (and even FPFT) underperforms some other variants in certain specific cases (e.g., Gemma on MATH). This discrepancy is likely attributed to differences between the training and test datasets. To provide further insights, we include loss curves in Appendix~\ref{lc-nlg}, where both HRP achieve significantly lower loss during fine-tuning compared to other LoRA variants (also for FPFT). 
Additionally, we report the computational cost in Table~\ref{cost}. As shown, HRP incurs negligible additional time compared to the subsequent fine-tuning process, while requiring no extra GPU memory.

% To provide deeper insights into the training dynamics, we present the loss curves of fine-tuning Llama2 in Figure~\ref{loss-curve}. The visualization convincingly validates our theoretical analysis, demonstrating that HRP indeed achieves superior converged results compared with random initialization. Besides, the loss curves reveal that HRP also accelerates the convergence of classic LoRA, especially in the beginning stage of fine-tuning. Aside from Llama2, we also present the loss curves of fine-tuning other models in Appendix~\ref{nlg-more}. 

\section{Conclusion}

In this paper, we first theoretically show the important role of LoRA initialization for convergence, where LoRA is likely to achieve random low-rank results with random initialization, and achieve the best low-rank results with a well-formed initialization. 
Then, to address the problem, we propose HRP, an initialization algorithm that makes LoRA achieve better fine-tuned results. HRP utilizes a few steps of high-rank LoRA optimization for preheating and computes the main singular vectors of the preheated result as initialization for the main LoRA. 
We further evaluate the effectiveness of HRP through experiments on NLU and NLG tasks and various models, where HRP makes LoRA achieve better performance compared with other initialization strategies.

% \section*{Impact Statement}

% This paper presents work whose goal is to advance the initialization of LoRA. A potential impact of this paper is guiding practitioners on more effective initialization for fine-tuning deep learning models using LoRA. After thorough consideration and analysis, it can be firmly stated that this research has no ethical aspects that could raise concerns. 

\bibliography{neurips_2025}
\bibliographystyle{plain}

\appendix

\section{Proofs of Theorems}

In this section, we present the proof for the theorems above. 
We begin with some notation, lemmas, and their proofs. 

\subsection{Notations and Lemmas}
For LoRA, what really matters the performance is the multiplied adapters $\frac{\alpha}{r}B_t A_t^\top$, which we denote as $X_t$. To study the properties of $X_t$, we define the following auxiliary values:
\begin{align*}
    Y_t=\frac{\alpha}{r}A_t A_t^\top,~~~~~Z_t=\frac{\alpha}{r}B_t B_t^\top,~~~~~G_t=\frac{\alpha}{r} \nabla_W\mathcal{L}(W+\frac{\alpha}{r}B_t^\top A_t).
\end{align*}
Then flow~\ref{abflow} can be restated as 
\begin{align*}
    \begin{cases}\label{xflow}
        dX_t=[-\eta_B G_tY_t-\eta_A Z_tG_t]dt,\\
        dY_t=[-\eta_A G_t^\top X_t-\eta_A X_t^\top G_t]dt,\\
        dZ_t=[-\eta_B G_t X_t^\top-\eta_B X_t G_t^\top]dt.
    \end{cases}
\end{align*}
With these notations, $X_t,Y_t,Z_t$ has the following properties:
\begin{enumerate}
    \item $Y_t$ and $Z_t$ are symmetric and semi-positive. 
    \item $d\operatorname{Trace}(Y_t)=d\operatorname{Trace}(Z_t)$. 
    \item For zero+random initialization, $X_t=O_{b\times a}$. 
    \item For zero+random LSI, $Y_t=O_{a\times a}$. 
    \item For zero+random RSI, $Z_t=O_{b\times b}$. 
\end{enumerate}
Besides, we denote $\|\cdot\|_F$ as the Frobenius norm, $U_X\Sigma_XV_X$ as the SVD decomposition of $X$ where $\Sigma_X=\operatorname{diag}(\sigma_1(X),\sigma_2(X)\cdots,)$ with $\sigma_1(X)\geq \sigma_2(X)\geq\cdots$. 

\begin{lemma}
    \label{asymlora-close-form}
    Gradient flow of matrix factorization~\ref{matrix-factorization} with Asymmetric LoRA in LSI ($\eta_A=\eta,\eta_B=0,A_0=O_{a\times r}$), $X_t$ has the following closed form:
    \begin{align*}
        X_t=\left[I-e^{-\eta Z_0t}\right]M.
    \end{align*}
    For matrix factorization~\ref{matrix-factorization} with Asymmetric LoRA in RSI ($\eta_A=0,\eta_B=\eta,B_0=O_{b\times r}$), $X_t$ has the following closed form:
    \begin{align*}
        X_t=M\left[I-e^{-\eta Y_0t}\right].
    \end{align*}
    Here, $e^{\operatorname{matrix}}$ denotes the exponential operation on a matrix. 
\end{lemma}
\begin{proof}
    For Asymmetric LoRA in LSI, by taking $\eta_A=\eta,\eta_B=0$ into flow~\ref{xflow}, we get
    \begin{align*}
        \begin{cases}
            dX_t=[-\eta Z_tG_t]dt,\\
            dY_t=[-\eta G_t^\top X_t-\eta X_t^\top G_t]dt,\\
            dZ_t=O_{b\times b}.
        \end{cases}
    \end{align*}
    which means $Z_t\equiv Z_0$ for all $t$, and $\dot{X_t}=-\eta Z_0(X_t-M)$. Then $X_t$ has an analytic solution 
    \begin{align*}
        X_t=X_0+\left[I-e^{-\eta Z_0t}\right]M=\left[I-e^{-\eta Z_0t}\right]M.
    \end{align*}

    For Asymmetric LoRA in LSI, by taking $\eta_A=0,\eta_B=\eta$ into flow~\ref{xflow}, we get
    \begin{align*}
        \begin{cases}
            dX_t=[-\eta G_tY_t]dt,\\
            dY_t=O_{a\times a},\\
            dZ_t=[-\eta G_t X_t^\top-\eta X_t G_t^\top]dt.
        \end{cases}
    \end{align*}
    which means $Y_t\equiv Y_0$ for all $t$, and $\dot{X_t}=-\eta (X_t-M)Y_0$. Then $X_t$ has an analytic solution 
    \begin{align*}
        X_t=X_0+M\left[I-e^{-\eta Y_0}\right]=M\left[I-e^{-\eta Y_0}\right].
    \end{align*}
\end{proof}
\begin{lemma}\label{lossineq}
    For any matrix $M\in\mathbb{R}^{d_1\times d_2}$ and any orthogonal matrix $U\in\mathbb{R}^{d_1\times d_1}$ and $r\leq d_2$, for
    \begin{align*}
        X=U\begin{pmatrix}
            I_r\\&O_{(d_1-r)\times (d_1-r)}
        \end{pmatrix}U^\top M,
    \end{align*}
    we have
    \begin{align*}
        \|M-X\|_F^2=\|M\|_F^2-\|X\|_F^2.
    \end{align*}
\end{lemma}
\begin{proof}
    We have
    \begin{align*}
        \|M\|_F^2&=\|X\|_F^2+\|M-X\|_F^2+2\operatorname{Trace}(X^\top(M-X)).
    \end{align*}
    Then it is sufficient to prove $\operatorname{Trace}(X^\top X)=\operatorname{Trace}(X^\top M)$. In fact, we have
    \begin{align*}
        \operatorname{Trace}\left(X^\top X\right)
        &=\operatorname{Trace}\left(M^\top U\begin{pmatrix}
            I_r\\&O_{(d_1-r)\times (d_1-r)}
        \end{pmatrix}U^\top\cdot U\begin{pmatrix}
            I_r\\&O_{(d_1-r)\times (d_1-r)}
        \end{pmatrix}U^\top M\right)\\
        &=\operatorname{Trace}\left(M^\top U\begin{pmatrix}
            I_r\\&O_{(d_1-r)\times (d_1-r)}
        \end{pmatrix}U^\top M\right)\\
        &=\operatorname{Trace}\left(X^\top M\right).
    \end{align*}
\end{proof}
\begin{lemma}\label{sigma-order}
    For any matrix $M\in\mathbb{R}^{d_1\times d_2}$ and any orthogonal matrix $U\in\mathbb{R}^{d_1\times d_1}$, and all $i,r<\max(d_1,d_2)$, we have
    \begin{align*}
        \sigma_i\left(U\begin{pmatrix}
            I_r\\&O_{(d_1-r)\times (d_1-r)}
        \end{pmatrix}U^\top M\right)\leq\sigma_i(M).
    \end{align*}
\end{lemma}
\begin{proof}
    For all $v\in\mathbb{R}^{d_2}$, we have
    \begin{align*}
        &\left\|U\begin{pmatrix}
            I_r\\&O_{(d_1-r)\times (d_1-r)}
        \end{pmatrix}U^\top Mv\right\|_F^2 \\
        &= \left(U\begin{pmatrix}
            I_r\\&O_{(d_1-r)\times (d_1-r)}
        \end{pmatrix}U^\top Mv\right)^\top \left(U\begin{pmatrix}
            I_r\\&O_{(d_1-r)\times (d_1-r)}
        \end{pmatrix}U^\top Mv\right) \\
        &= v^\top M^\top U\begin{pmatrix}
            I_r\\&O_{(d_1-r)\times (d_1-r)}
        \end{pmatrix}U^\top U\begin{pmatrix}
            I_r\\&O_{(d_1-r)\times (d_1-r)}
        \end{pmatrix}U^\top Mv \\
        &= v^\top M^\top U\begin{pmatrix}
            I_r\\&O_{(d_1-r)\times (d_1-r)}
        \end{pmatrix}U^\top Mv \\
        &= \sum_{i=1}^r v^\top M^\top u_i u_i^\top Mv \\
        &= \left(\sum_{i=1}^r v^\top M^\top u_i \right)^2 \\
        &\leq \left(\sum_{i=1}^{d_1} v^\top M^\top u_i \right)^2 \\
        &=\|Mv\|_F^2.
    \end{align*}
    Then for all space $S\subset\mathbb{R}^{d_2}$, we have
    \begin{align*}
        \min_{v\in S}\left\|U\begin{pmatrix}
            I_r\\&O_{(d_1-r)\times (d_1-r)}
        \end{pmatrix}U^\top Mv\right\|_F\leq \min_{v\in S}\|Mv\|_F. 
    \end{align*}
    Thus we have
    \begin{align*}
        \sigma_i\left(U\begin{pmatrix}
            I_r\\&O_{(d_1-r)\times (d_1-r)}
        \end{pmatrix}U^\top M\right) &= \max_{\substack{S \subset \mathbb{R}^{d_2} \\ \operatorname{dim}(S) = i}}\min_{v\in S}\left\|U\begin{pmatrix}
            I_r\\&O_{(d_1-r)\times (d_1-r)}
        \end{pmatrix}U^\top Mv\right\|_F\\
        &\leq \max_{\substack{S \subset \mathbb{R}^{d_2} \\ \operatorname{dim}(S) = i}}\min_{v\in S}\left\|Mv\right\|_F\\
        &=\sigma_i(M).
    \end{align*}
\end{proof}
\begin{lemma}[Gronwall's Theorem~\cite{howard1998gronwall}]
    \label{gronwall}
    Consider $x_t\geq 0$ and inequality
    \begin{align*}
        dx_t\leq [ax_t+f(t)]dt,~~~~~x_0=0. 
    \end{align*}
    We have
    \begin{align*}
        x_t\leq e^{at}\int_0^tf(s)e^{-as}ds.
    \end{align*}
\end{lemma}
\begin{proof}
    Consider $x_te^{-at}\geq 0$ and
    \begin{align*}
        d\left[x_te^{-at}\right]\leq e^{-at}[ax_t+f(t)]dt-ax_te^{-at}dt=f(t)e^{-at}dt,
    \end{align*}
    thus we get
    \begin{align*}
        x_te^{-at}&\leq x_0+\int_0^tf(s)e^{-as}ds,\\
        x_t&\leq e^{at}\int_0^tf(s)e^{-as}ds.
    \end{align*}
\end{proof}

\subsection{Proof for Theorem~\ref{ms-loss}}\label{ms-loss-proof}
In Theorem~\ref{ms-loss}, the objective is $\mathcal{L}_t$, which refers to $\frac{1}{2}\left\|\frac{\alpha}{r}B_t A_t^\top-M\right\|^2$. Equivalently, it has the following restatement version. 
\begin{theorem}[Restatement of theorem~\ref{ms-loss}]
    Consider Asymmetric LoRA under objective~\ref{matrix-factorization} in gradient flow~\ref{abflow} with the frozen adapter from Gaussian initialization or orthogonal initialization. For LSI and RSI, we have for all $t>0$:
    \begin{align*}
        \mathbb{E}_{LSI}\left[\|X_t-M\|_F^2\right]\geq \frac{b-r}{b}\sum_{i=1}^{\min\{a,b\}}\sigma_i(M)^2,~~~~
        \mathbb{E}_{RSI}\left[\|X_t-M\|_F^2\right]\geq \frac{a-r}{a}\sum_{i=1}^{\min\{a,b\}}\sigma_i(M)^2,
    \end{align*}
    where $\mathbb{E}$ represents the expectation with respect to randomness in initialization. The inequality becomes an equality when $t\to\infty$. 
\end{theorem}
\begin{proof}
    For Asymmetric LoRA in LSI with fixed $B_0$, according to Lemma~\ref{asymlora-close-form} we have
    \begin{align*}
        X_t=\left[I-e^{-\eta Z_0t}\right]M=\left[I-e^{-\eta B_0B_0^\top t}\right]M.
    \end{align*}
    For the SVD decomposition $B_0=U_B\Sigma_BV_B^\top$, we have $Z_0=U_B(\Sigma_B)^2U_B^\top$ where $\Sigma_B$ is a diagonal matrix with only first $r$ elements non-zero. Then, consider $t\to\infty$, we have
    \begin{align}
        \lim_{t\to\infty}X_t&=\lim_{t\to\infty}\left[\left[I-e^{-\eta Z_0t}\right]M\right]\\
        &=\left[I-\lim_{t\to\infty}e^{-\eta Z_0t}\right]M\\
        &=U_B\left[I-\lim_{t\to\infty}e^{-\eta \Sigma_B^2t}\right]U_B^\top M=U_B\begin{pmatrix}
            I_r\\&O_{b-r}
        \end{pmatrix}U_B^\top M.
    \end{align}
    According to Lemma~\ref{lossineq}, we have
    \begin{align*}
        \lim_{t\to\infty}\|X_t-M\|_F^2 &= \|\lim_{t\to\infty}X_t-M\|_F^2 = \|M\|_F^2 - \left\|\lim_{t\to\infty}X_t\right\|_F^2\\
        &= \|M\|_F^2 - \left\|U_B\begin{pmatrix}
            I_r\\&O_{b-r}
        \end{pmatrix}U_B^\top M\right\|_F^2 \\
        &= \|M\|_F^2 - \left\|\begin{pmatrix}
            I_r\\&O_{b-r}
        \end{pmatrix}U_B^\top M\right\|_F^2 \\
        &= \|M\|_F^2 - \sum_{i=1}^{r}\left\|u_i(B)^\top M\right\|_F^2.
    \end{align*}
    For finite $t$, we have
    \begin{align*}
        \frac{d\|X_t-M\|_F^2}{dt} &= 2\operatorname{Trace}\left([X_t-M]^\top\frac{d}{dt}[X_t-M]\right) \\
        &= 2\operatorname{Trace}\left([X_t-M]^\top[-\eta Z_0G_t]\right) \\
        &= -2\frac{\eta\alpha^2}{r^2}\operatorname{Trace}\left([X_t-M]^\top B_0B_0^\top [X_t-M]\right)\leq 0, 
    \end{align*}
    indicating for any fixed $B_0$, $\|X_t-M\|_F^2$ is lower bounded by $\lim_{t\to\infty}\|X_t-M\|_F^2$. 

    Gaussian initialization and orthogonal initialization share the same probability for the same $U_B$, thus their properties at converged results are the same. 
    For any orthogonal matrix $U$, denote $U^{(i)}$ as the matrix that moves the right $i$ column of $U$ to left, i.e. $[U_{[:,:i]},U_{[:,i:]}]$ in Python. 
    Due to Gaussian initialization of $B$, the p.d.f at $U_B=U^{(i)}$ is the same as p.d.f at $U_B=U^{(j)}$ for all $i, j$. Thus, we have
    \begin{align*}
        \mathbb{E}_{B_0} \lim_{t\to\infty}\|X_t-M\|_F^2 &=
        \mathbb{E}_{B_0} \left[\|M\|_F^2 - \sum_{i=1}^{r}\left\|u_i(B)^\top M\right\|_F^2\right] ,\\
        &= \|M\|_F^2 - \mathbb{E}_U\left.\frac{1}{b}\sum_{j=1}^{b} \sum_{i=1}^{r}\left\|u_i(B)^\top M\right\|_F^2\right|_{U_B=U^{(j)}},\\
        &= \|M\|_F^2 - \mathbb{E}_U\frac{1}{b}\sum_{i=1}^{r}\sum_{j=1}^{b}\left\|u_i(U^{(j)})^\top M\right\|_F^2,\\
        &= \|M\|_F^2 - \frac{1}{b}\sum_{i=1}^{r}\mathbb{E}_U\left\|U^\top M\right\|_F^2,\\
        &= \|M\|_F^2 - \frac{r}{b}\left\|M\right\|_F^2=\frac{b-r}{b}\sum_{i=1}^{\min\{a,b\}}\sigma_i(M)^2.
    \end{align*}
    For finite $t$, we have
    \begin{align*}
        \mathbb{E}_{B_0} \|X_t-M\|_F^2 \geq \mathbb{E}_{B_0} \lim_{t\to\infty}\|X_t-M\|_F^2 =\frac{b-r}{b}\sum_{i=1}^{\min\{a,b\}}\sigma_i(M)^2.
    \end{align*}

    Asymmetric LoRA in RSI with $X_t$ approximating $M$, is equivalent with Asymmetric LoRA in LSI with $X_t^\top$ approximating $M^\top$. Thus for RSI we have
    \begin{align*}
        \mathbb{E}_{A_0} \|X_t-M\|_F^2 = \mathbb{E}_{B_0} \|X_t^\top-M^\top\|_F^2  \geq \mathbb{E}_{B_0} \lim_{t\to\infty}\|X_t^\top-M^\top\|_F^2 =\frac{a-r}{a}\sum_{i=1}^{\min\{a,b\}}\sigma_i(M)^2.
    \end{align*}
\end{proof}

\subsection{Proof for Theorem~\ref{lora-bad}}\label{lora-bad-proof}
In Theorem~\ref{lora-bad}, the assumption $A_0^\top v_i(M)=O_a$ means $Y_0 v_i(M)=O_a$ while $B_0^\top u_i(M)=O_b$ means $Z_0 u_i(M)=O_b$. In conclusion, $\frac{\alpha}{r}B_tA_t^\top$ can be changed to $X_t$. Equivalently, it has the following restatement version. 
\begin{theorem}[Restatement of Theorem~\ref{lora-bad}]
    For objective~\ref{matrix-factorization} in gradient flow~\ref{xflow}, if there exists $i\leq r$ making the initial $A_0$ and $B_0$ satisfy $Y_0 v_i(M)=O_a$ and $Z_0 u_i(M)=O_b$, we have for any $t$:
    \begin{align*}
        X_t v_i=O_b,~~~\text{and}~~~X_t^\top u_i =O_a,
    \end{align*}
    resulting in for any $t$:
    \begin{align*}
        \mathcal{L}_t-\mathcal{L}^*\geq \frac{1}{2}[\sigma_i(M)^2-\sigma_{r+1}(M)^2]>0,
    \end{align*}
    where $\sigma_i(M)$ is the $i$-th large singular value of $M$ and $u_i(M),v_i(M)$ are the coresponding left and right singular vector. 
\end{theorem}
\begin{proof}
    Due to the semi-positive property of $Y_0$ and $Z_0$, we have for $t=0$:
    \begin{align}\label{lorabad-base}
        X_tv_i=O_b,~~~~Y_tv_i=O_a,~~~~X_t^\top u_i=O_a,~~~~Z_t^\top u_i=O_b. 
    \end{align}
    We then prove~\ref{lorabad-base} true for all $t>0$. According to induction, it is sufficient to prove each gradient in~\ref{lorabad-base} to be zero for all $t$ that satisfies equations~\ref{lorabad-base}. In fact, we have
    \begin{align*}
        \frac{d(X_tv_i)}{dt}&=-\frac{\eta\alpha}{r} Z_t(X_t-M)v_i-\frac{\eta\alpha}{r} (X_t-M)Y_tv_i= \frac{\eta\alpha}{r} Z_tMv_i=\frac{\eta\alpha}{r}\sigma_i(M) Z_t^\top u_i=O_b,\\
        \frac{d(Y_tv_i)}{dt}&=-\frac{\eta\alpha}{r} X_t^\top (X_t-M)v_i-\frac{\eta\alpha}{r} (X_t-M)^\top X_tv_i=\frac{\eta\alpha}{r} X_t^\top Mv_i= \frac{\eta\alpha}{r}\sigma_i(M) X_t^\top u_i=O_a,\\
        \frac{d(X_t^\top u_i)}{dt}&=-\frac{\eta\alpha}{r} \left[Z_t(X_t-M)\right]^\top u_i-\frac{\eta\alpha}{r} \left[(X_t-M)Y_t\right]^\top u_i=\frac{\eta\alpha}{r} Y_tM^\top u_i= \frac{\eta\alpha}{r}\sigma_i(M) Y_tv_i=O_a,\\
        \frac{d(Z_tu_i)}{dt}&=-\frac{\eta\alpha}{r} X_t(X_t-M)^\top u_i -\frac{\eta\alpha}{r} (X_t-M)X_t^\top u_i=\frac{\eta\alpha}{r} X_tM^\top u_i=\frac{\eta\alpha}{r}\sigma_i(M) X_t v_i=O_b.
    \end{align*}
    This means that the composition $\sigma_i(M)u_i(M)v_i(M)^\top$ will not emerge from $X_t$ as $t$ grows up. However, the best low-rank approximation of $M$ requires $\sigma_i(M)u_iv_i^\top$ to be included, which makes the optimization never achieve the best low-rank result. For loss, consider $M'=M-\sigma_i(M)u_iv_i^\top$, we have
    \begin{align*}
        \mathcal{L}_t&=\frac{1}{2}\|X_t-M\|_F^2\\
        &=\frac{1}{2}\operatorname{Trace}((X_t-M'-\sigma_i(M)u_iv_i^\top)(X_t-M'-\sigma_i(M)u_iv_i^\top)^\top)\\
        &=\frac{1}{2}\|X_t-M'\|_F^2+\frac{1}{2}\|\sigma_i(M)u_iv_i^\top\|_F^2\\
        &\geq \frac{1}{2}\sum_{j=r+2}^{\min\{a,b\}}\sigma_j(M)^2+\frac{1}{2}\sigma_i(M)^2\\
        &=\mathcal{L}^*+\frac{1}{2}\left[\sigma_i(M)^2-\sigma_{r+1}(M)^2\right]
    \end{align*}
\end{proof}

\subsection{Proof for Theorem~\ref{lora-similar}}\label{lora-similar-proof}
In the proof of Theorem~\ref{lora-similar}, we denote all ``tilde'' variables as those from Asymmetric LoRA, while the without ``tilde'' as those from classic LoRA. Same problem means a same $W\to\nabla_W\mathcal{L}$ mapping. 
By replacing $\frac{\alpha}{r}B_tA_t^\top$ with $X_t$ and its ``tilde'' version, Theorem~\ref{lora-similar} has the following restatement version. 
\begin{theorem}[Restatement of theorem~\ref{lora-similar}]
    Consider the gradient flow of classic LoRA $X_t$ and the gradient flow of Asymmetric LoRA $\tilde{X_t}$ in the same zero+random initialization. Assume $G_t$ is bounded by $R$ and is Lipschitz in the Frobenius norm, then the difference between Asymmetric LoRA and classic LoRA is upper bounded by
    \begin{align*}
        \left\|X_t - \tilde{X_t}\right\|_F=O\left(\eta^4R^3rt^4\right),
    \end{align*}
    for small $t$ when they are under a same gradient calculator. 
\end{theorem}
\begin{proof}
    For classic LoRA in gradient flow, we have
    \begin{align*}
        \|X_t\|_F^2 &= \|\frac{\alpha}{r}B_tA_t^\top \|_F^2 \\
        &\leq \frac{\alpha^2}{r^2}\|B_t\|_F^2\|A_t\|_F^2 \\
        &= \frac{\alpha^2}{r^2}tr(B_tB_t^\top)tr(A_tA_t^\top) \\
        &= tr(Z_t)tr(Y_t) \\
        &\leq r\|Z_t\|_F\|Y_t\|_F
    \end{align*}
    Consider the flow of $x_t=\|Y_t-Y_0\|_F+\|Z_t-Z_0\|_F$, we have
    \begin{align*}
        \frac{dx_t}{dt} &\leq \left\|\frac{d(Y_t-Y_0)}{dt}\right\|_F + \left\|\frac{d(Z_t-Y_0)}{dt}\right\|_F \\
        &\leq 4\eta \|X_t\|_F\|G_t\|_F \\
        &\leq 4\eta R\sqrt{r\|Z_t\|_F\|Y_t\|_F} \\
        &\leq 4\eta R\sqrt{r}\sqrt{\|Z_t-Z_0\|_F\|Y_t-Y_0\|_F+\|Z_0\|_F\|Y_t-Y_0\|_F+\|Z_t-Z_0\|_F\|Y_0\|_F} \label{nt1}\\
        &\leq 4\eta R\sqrt{r}\left(x_t+2\sqrt{\|Z_0\|_F+\|Y_0\|_F}\sqrt{x_t}\right),
    \end{align*}
    where inequality~\ref{nt1} holds because $\|Z_0\|_F\|Y_0\|_F=0$ in zero+random initialization strategy. Besides, we have $x_t>0$ for $t>0$, which means
    \begin{align*}
        \frac{d\sqrt{x_t}}{dt}\leq 2\eta R\sqrt{r}\left(\sqrt{x_t}+2\sqrt{\|Z_0\|_F+\|Y_0\|_F}\right).
    \end{align*}
    Through Gronwall Theorem, we have
    \begin{align*}
        \sqrt{x_t} &\leq e^{2\eta R\sqrt{r}t}\int_0^t 4\eta R\sqrt{r}\sqrt{\|Z_0\|_F+\|Y_0\|_F}e^{-2\eta R\sqrt{r}s}ds\\
        &=2\sqrt{\|Z_0\|_F+\|Y_0\|_F}e^{2\eta R\sqrt{r}t}\left(1-e^{-2\eta R\sqrt{r}t}\right)\\
        \|Y_t-Y_0\|_F+\|Z_t-Z_0\|_F &\leq 2\left(\|Z_0\|_F+\|Y_0\|_F\right)\left(e^{2\eta R\sqrt{r}t}-1\right)^2.
    \end{align*}
    Considering the difference to gradient flow of Asymmetric LoRA $\tilde{X_t}$, we have
    \begin{align*}
        \frac{d\left\|X_t-\tilde{X_t}\right\|_F}{dt}&\leq \left\|\frac{d[X_t-\tilde{X_t}]}{dt}\right\|_F\\
        &\leq \left\|-\eta(Z_t-Z_0)G_t-\eta Z_0(G_t-\tilde{G_t})-\eta G_t(Y_t-Y_0)-\eta(G_t-\tilde{G_t})Y_0\right\|_F\\
        &\leq \eta\left[\|G_t\|_F(\|Y_t-Y_0\|_F+\|Z_t-Z_0\|_F) + \left(\|Z_0\|_F+\|Y_0\|_F\right)\|G_t-\tilde{G_t}\|_F\right]\\
        &\leq \eta\left[2R\left(\|Z_0\|_F+\|Y_0\|_F\right)\left(e^{2\eta R\sqrt{r}t}-1\right)^2 + L\left(\|Z_0\|_F+\|Y_0\|_F\right)\|X_t-\tilde{X_t}\|_F\right].
    \end{align*}
    Through Gronwall Theorem, we have
    \begin{align*}
        \left\|X_t-\tilde{X_t}\right\|_F &\leq \int_0^t 2\eta R e^{\eta L\left(\|Z_0\|_F+\|Y_0\|_F\right)(t-s)} \left(\|Z_0\|_F+\|Y_0\|_F\right) \left(e^{2\eta R\sqrt{r}s}-1\right)^2 ds \\
        &\leq 2\eta Rt e^{\eta L\left(\|Z_0\|_F+\|Y_0\|_F\right)t} \left(\|Z_0\|_F+\|Y_0\|_F\right) \left(e^{2\eta R\sqrt{r}t}-1\right)^2\\
        &= O\left(\eta^4R^3rt^4\right)
    \end{align*}

    % \begin{align*}
    %     dB_t^\top B_t &= -\eta B_t^\top G_tA_t - \eta A_tG_t^\top B_t = dA_t^\top A_t \\
    %     \|Y_0\|_F^2 = \|B_t^\top B_t - A_t^\top A_t\|_F^2 &= tr(B_t^\top B_tB_t^\top B_t + A_t^\top A_tA_t^\top A_t - 2B_t^\top B_tA_t^\top A_t) \\
    %     &= \|Y_t\|_F^2+\|Z_t\|_F^2-2\|X_t\|_F^2 \\
    %     &\geq \|Y_t\|_F^2-2\|X_t\|_F^2 \\
    %     &= \|Y_t-Y_0+Y_0\|_F^2-2\|X_t\|_F^2 \\
    %     &\geq \|Y_0\|_F^2 + 2tr(Y_0^\top (Y_t-Y_0)) - 2\|X_t\|_F^2 \\
    %     R &\geq tr(Y_0^\top (Y_t-Y_0))
    % \end{align*}

    % \begin{align*}
    %     dB_t^\top B_t &= -\eta B_t^\top G_tA_t - \eta A_tG_t^\top B_t = dA_t^\top A_t \\
    %     tr(Y_0)^2 &= tr(B_t^\top B_t - A_t^\top A_t)^2 \\
    %     &= tr(B_t^\top B_t)^2 + tr(A_t^\top A_t)^2 - 2tr(B_t^\top B_t)tr(A_t^\top A_t) \\
    %     &= tr(Y_t)^2+tr(Z_t)^2-2tr(Y_t)tr(Z_t) \\
    %     &= tr(Y_0)^2+tr(Y_t-Y_0)^2+2tr(Y_0)tr(Y_t-Y_0)-2tr(Y_t)tr(Z_t) \\
    %     &= tr(Y_0)^2-tr(Z_t)^2+2tr(Y_0)tr(Z_t)-2tr(Y_0)tr(Z_t) 
    % \end{align*}

\end{proof}

\subsection{Proof for Theorem~\ref{asym-wise}}\label{asym-wise-proof}
In the Theorem~\ref{asym-wise}, by replacing $A_0$ initialization to $Y_0$ initialization and $B_0$ initialization to $Z_0$ initialization, it has the following restatement version. 
\begin{theorem}[Restatement of theorem~\ref{asym-wise}]\label{asym-wise-restate}
    Consider Asymmetric LoRA under objective~\ref{matrix-factorization} in gradient flow~\ref{xflow}. For RSI with $Y_0=\frac{\alpha}{r}V_M\operatorname{diag}\left(I_r,O_{(a-r)\times (a-r)}\right)V_M^\top, Z_0=O_{b\times b}$ or LSI with $Y_0=O_{a\times a},Z_0=\frac{\alpha}{r}U_M\operatorname{diag}\left(I_r,O_{(b-r)\times (b-r)}\right)U_M^\top$, we have
    \begin{align*}
        \mathcal{L}_t-\mathcal{L}^*=O\left(\exp\left\{-\frac{2\alpha\eta}{r} t\right\}\right).
    \end{align*}
\end{theorem}
\begin{proof}
    According to Lemma~\ref{asymlora-close-form}, with LSI initialization, we have 
    \begin{align*}
        X_t=M[I-e^{-\eta Z_0 t}]=(1-e^{-\frac{\alpha\eta}{r} t})MV_M\begin{pmatrix}
            I_r\\&O_{(a-r)\times (a-r)}
        \end{pmatrix}V_M^\top,
    \end{align*}
    which has loss
    \begin{align*}
        \mathcal{L}(X_t)&=\frac{1}{2}\|X_t-M\|_F^2\\
        &=\frac{1}{2}\left\|M-(1-e^{-\frac{\alpha\eta}{r} t})MV_M\begin{pmatrix}
            I_r\\&O_{(a-r)\times (a-r)}
        \end{pmatrix}V_M^\top\right\|_F^2\\
        &=\frac{1}{2}\left\|U_M\Sigma_MV_M^\top -(1-e^{-\frac{\alpha\eta}{r} t})U_M\Sigma_M\begin{pmatrix}
            I_r\\&O_{(a-r)\times (a-r)}
        \end{pmatrix}V_M^\top \right\|_F^2\\
        &=\frac{1}{2}\left\|\Sigma_M-(1-e^{-\frac{\alpha\eta}{r} t})\Sigma_M\begin{pmatrix}
            I_r\\&O_{(a-r)\times (a-r)}
        \end{pmatrix}\right\|_F^2\\
        &=\frac{1}{2}\left\|\Sigma_M\begin{pmatrix}
            e^{-\frac{\alpha\eta}{r} t}I_r\\&I_{a-r}
        \end{pmatrix}\right\|_F^2\\
        &=\frac{1}{2}e^{-\frac{2\alpha\eta}{r} t}\sum_{i=1}^r\sigma_i(M)^2+\frac{1}{2}\sum_{i=r+1}^{\min\{a,b\}}\sigma_i(M)^2\\
        &=O\left(\exp\left\{-\frac{2\alpha\eta}{r} t\right\}\right)+\mathcal{L}^*.
    \end{align*}
    Asymmetric LoRA in RSI with $X_t$ approximating $M$, is equivalent with Asymmetric LoRA in LSI with $X_t^\top$ approximating $M^\top$. Thus for RSI we also have
    \begin{align*}
        \mathcal{L}_t-\mathcal{L}^*=O\left(\exp\left\{-\frac{2\alpha\eta}{r} t\right\}\right).
    \end{align*}
\end{proof}

\subsection{Proof for Theorem~\ref{lora-wise}}\label{lora-wise-proof}
In the Theorem~\ref{lora-wise}, by replacing $A_0$ initialization to $Y_0$ initialization and $B_0$ initialization to $Z_0$ initialization, it has the following restatement version. 
\begin{theorem}[Restatement of theorem~\ref{lora-wise}]
    Consider classic LoRA under objective~\ref{matrix-factorization} in gradient flow~\ref{xflow}. For initialization same as Theorem \ref{asym-wise-restate}, we have
    \begin{align*}
        \mathcal{L}_t-\mathcal{L}^*=O\left(\exp\left[\frac{-\alpha\eta t}{r}\left(1+\sqrt{1+\frac{2r^2}{\alpha^2} \sigma_r(M)^2\left[1-e^{-\frac{\alpha\eta t}{2r}}\right]^2 }\right)\right]\right).
    \end{align*}
    where $\sigma_i(M)$ is the $i$-th singular value of $M$ and is defined in Theorem~\ref{lora-bad}.
\end{theorem}
\begin{proof}
    We consider
    \begin{align}
        \hat{X_t}:=U_M^\top X_tV_M,~~~~\hat{Y_t}:=V_M^\top Y_tV_M,~~~~\hat{Z_t}:=U_M^\top Z_tU_M,
    \end{align}
    and first prove 
    \begin{align}
        \forall r, \hat{X_t},\hat{Y_t},\hat{Z_t}~\text{are diagonal matrices with only the first $r$ elements non-zero.}\label{lora-wise-cl}
    \end{align}
    If initialized with RSI, i.e. $Y_0=\frac{\alpha}{r}V_M\operatorname{diag}\left(I_r,O_{(a-r)\times (a-r)}\right)V_M^\top, Z_0=O_{b\times b}$, then we have 
    \begin{align*}
        \hat{X_t}:=U_M^\top X_0V_M=O_{b\times a},~~~~\hat{Y_t}:=V_M^\top Y_0V_M=\frac{\alpha}{r}\begin{pmatrix}
            I_r\\&O_{a-r}
        \end{pmatrix},~~~~\hat{Z_t}:=U_M^\top Z_0U_M=O_{b\times b}.
    \end{align*}
    If initialized with LSI, i.e. $Y_0=O_{a\times a},Z_0=\frac{\alpha}{r}U_M\operatorname{diag}\left(I_r,O_{(b-r)\times (b-r)}\right)U_M^\top$, then we have 
    \begin{align*}
        \hat{X_t}:=U_M^\top X_0V_M=O_{b\times a},~~~~\hat{Y_t}:=V_M^\top Y_0V_M=O_{a\times a},~~~~\hat{Z_t}:=U_M^\top Z_0U_M=\frac{\alpha}{r}\begin{pmatrix}
            I_r\\&O_{b-r}
        \end{pmatrix}.
    \end{align*}
    Thus, claim~\ref{lora-wise-cl} is true for $t=0$. If it is true for $t$, then their gradient satisfies
    \begin{align*}
        \dot{\hat{X_t}}&=U_M^\top\left[-\eta Z_t(X_t-M)-\eta (X_t-M)Y_t\right]V_M=-\eta \hat{Z_t}(\hat{X_t}-\Sigma_M)-\eta (\hat{X_t}-\Sigma_M)\hat{Y_t},\\
        \dot{\hat{Y_t}}&=V_M^\top\left[-\eta X_t^\top (X_t-M)-\eta (X_t-M)^\top X_t\right]V_M= -\eta \hat{X_t}^\top (\hat{X_t}-\Sigma_M)-\eta (\hat{X_t}-\Sigma_M)^\top \hat{X_t},\\
        \dot{\hat{Z_t}}&=U_M^\top\left[-\eta X_t(X_t-M)^\top -\eta (X_t-M)X_t^\top\right]U_M= -\eta \hat{X_t}(\hat{X_t}-\Sigma_M)^\top -\eta (\hat{X_t}-\Sigma_M)\hat{X_t}^\top,
    \end{align*}
    which are all diagonal matrices with only the first $r$ elements non-zero. Thus, claim~\ref{lora-wise-cl} is true for all $t$. We denote the $i$-th diagonal elements for $\hat{X_t},\hat{Y_t},\hat{Z_t}$ as $x_{t,i},y_{t,i},z_{t,i}$. Then for each $i$, we have
    \begin{align*}
        \dot{x_{t,i}}&=-\eta z_{t,i}(x_{t,i}-\sigma_i(M))-\eta (x_{t,i}-\sigma_i(M))y_{t,i}=-\eta(y_{t,i}+z_{t,i})(x_{t,i}-\sigma_i(M)),\\
        \dot{y_{t,i}}=\dot{z_{t,i}}&= -\eta x_{t,i}(x_{t,i}-\sigma_i(M)) -\eta (x_{t,i}-\sigma_i(M))x_{t,i}=-2\eta x_{t,i}(x_{t,i}-\sigma_i(M)),
    \end{align*}
    and they are independent with other $j\neq i$. 
    Consider
    \begin{align*}
        d\left[(y_{t,i}+z_{t,i})^2\right]&=-4\eta(y_{t,i}+z_{t,i})x_{t,i}(x_{t,i}-\sigma_i(M))\\
        &=-4\eta x_{t,i}(y_{t,i}+z_{t,i})(x_{t,i}-\sigma_i(M))=2d\left[x_{t,i}^2\right],
    \end{align*}
    resulting in
    \begin{align*}
        y_{t,i}+z_{t,i}&=\sqrt{(y_{0,i}+z_{0,i})^2+2x_{t,i}^2-2x_{0,i}^2}=\sqrt{\frac{\alpha^2}{r^2}+2x_{t,i}^2}. 
    \end{align*}
    So, the dynamic of $x_{t,i}$ is
    \begin{align*}
        \dot{x_{t,i}}&=-\eta(y_{t,i}+z_{t,i})(x_{t,i}-\sigma_i(M))=-\eta(x_{t,i}-\sigma_i(M))\sqrt{\frac{\alpha^2}{r^2}+2x_{t,i}^2}.
    \end{align*}
    For a fixed $i$, consider another flow $x_t$ with some $s$, which satisfies
    \begin{align*}
        x_0 &= x_{0,i} \\
        \dot{x_t} &= -\eta(x_t-\sigma_i(M))\frac{\alpha}{r},~~~0<t<s, \\
        \dot{x_t} &= -\eta(x_t-\sigma_i(M))\sqrt{\frac{\alpha^2}{r^2}+2x_s},~~~s<t. 
    \end{align*}
    Then we have always $x_{t,i}\geq x_i$ and $(\sigma_i(M)-x_{t,i})^2\leq (\sigma_i(M)-x_t)^2$. 
    We can get $x_s=\sigma_i(M)\left[1-e^{-\frac{\alpha\eta s}{r}}\right]$, and for $t>s$:
    \begin{align*}
        x_t &= \sigma_i(M)\left[1-\exp\left[\frac{-\alpha\eta s}{r}-\eta(t-s)\sqrt{\frac{\alpha^2}{r^2}+2x_s^2}\right]\right].
    \end{align*}
    By setting $t=2s$, we have
    \begin{align*}
        x_{2s,i}\geq x_{2s} &= \sigma_i(M)\left[1-\exp\left[\frac{-\alpha\eta s}{r}\left(1+\sqrt{1+\frac{2r^2}{\alpha^2}x_s^2}\right)\right]\right] \\
        &= \sigma_i(M)\left[1-\exp\left[\frac{-\alpha\eta s}{r}\left(1+\sqrt{1+\frac{2r^2}{\alpha^2} \sigma_i(M)^2\left[1-e^{-\frac{\alpha\eta s}{r}}\right]^2 }\right)\right]\right] \\
        &\geq \sigma_i(M)\left[1-\exp\left[\frac{-\alpha\eta s}{r}\left(1+\sqrt{1+\frac{2r^2}{\alpha^2} \sigma_r(M)^2\left[1-e^{-\frac{\alpha\eta s}{r}}\right]^2 }\right)\right]\right].
    \end{align*}
    So, for $X_t$, we have
    \begin{align*}
        \mathcal{L}_t-\mathcal{L}^*&=\frac{1}{2}\|U_M\hat{X_t}V_M^\top-M\|_F^2-\mathcal{L}^*\\
        &=\|\hat{X_t}-\Sigma_M\|_F^2-\mathcal{L}^*\\
        &=\sum_{i=1}^r(\sigma_i(M)-x_{t,i})^2\\
        &\leq \sum_{i=1}^r \sigma_i(M)^2 \exp\left[\frac{-\alpha\eta t}{2r}\left(1+\sqrt{1+\frac{2r^2}{\alpha^2} \sigma_r(M)^2\left[1-e^{-\frac{\alpha\eta t}{2r}}\right]^2 }\right)\right]^2\\
        &=O\left(\exp\left[\frac{-\alpha\eta t}{r}\left(1+\sqrt{1+\frac{2r^2}{\alpha^2} \sigma_r(M)^2\left[1-e^{-\frac{\alpha\eta t}{2r}}\right]^2 }\right)\right]\right) \\
    \end{align*}
    
\end{proof}

\subsection{Proof for Theorem~\ref{hrp-theory}}\label{hrp-theory-proof}
We prove Theorem~\ref{hrp-theory} in the notation of $X_t,Y_t,Z_t$. First, we restate this theorem and HRP algorithm. 
\begin{theorem}[Restatement of theorem~\ref{hrp-theory}]
    For one LSI Asymmetric LoRA $\hat{X_t}$ with rank $\operatorname{hrp\_rank}\geq r$ and random orthogonal initialization 
    $$\hat{A}_0=O_{a\times r}, \hat{B}_0=\sum_{i=1}^{\operatorname{hrp\_rank}}\hat{u_i}e_{i,\operatorname{hrp\_rank}},$$
    And another RSI Asymmetric LoRA $X_s$ with initialization 
    \begin{align*}
        \hat{A}_0=\sum_{i=1}^{\operatorname{hrp\_rank}}v_i(\hat{X_t})e_{i,\operatorname{hrp\_rank}}^\top, \hat{B}_0=O_{b\times r}.
    \end{align*}
    We have for all $t$, $X_s$ converges to $X_\infty$ with loss in expectation
    \begin{align*}
        \mathbb{E}_{U}\mathcal{L}_{\infty}\leq\sum_{i=r+1}^{\operatorname{hrp\_rank}}\sigma_i(M)^2+\frac{a-\operatorname{hrp\_rank}}{2a}\sum_{i=1}^{\max(a,b)}\sigma_i(M)^2.
    \end{align*}
\end{theorem}
\begin{proof}
    According to Lemma~\ref{asymlora-close-form}, for the preheating orth-init LSI we have
    \begin{align*}
        \hat{X_t}&=\left[I-e^{-\eta \hat{Z_0}t}\right]M\\
        &=(1-e^{-\frac{\alpha\eta t}{r}})U_{\hat{Z_0}}\begin{pmatrix}
            I_{\operatorname{hrp\_rank}}\\&O_{(b-\operatorname{hrp\_rank})\times(b-\operatorname{hrp\_rank})}
        \end{pmatrix}U_{\hat{Z_0}}^\top M\\
        &=(1-e^{-\frac{\alpha\eta t}{r}})\hat{X_\infty}.
    \end{align*}
    For the following HRP derived RSI, we have
    \begin{align*}
        X_t &= M\left[I-e^{-\eta \hat{Y_0}t}\right] \\
        &= (1-e^{-\frac{\alpha\eta t}{r}})MV_{\hat{X_t}}\begin{pmatrix}
            I_r\\&O_{(a-r)\times(a-r)}
        \end{pmatrix}V_{\hat{X_t}}^\top \\
        &= (1-e^{-\frac{\alpha\eta t}{r}})MV_{\hat{X_\infty}}\begin{pmatrix}
            I_r\\&O_{(a-r)\times(a-r)}
        \end{pmatrix}V_{\hat{X_\infty}}^\top \\
        &= \left(1-e^{-\frac{\alpha\eta t}{r}}\right)X_\infty.
    \end{align*}
    According to Lemma~\ref{lossineq}, we have
    \begin{align*}
        \|X_t\|_F^2 &\geq \left\|U_{\hat{Z_0}}\begin{pmatrix}
            I_{\operatorname{hrp\_rank}}\\&O_{(b-\operatorname{hrp\_rank})\times(b-\operatorname{hrp\_rank})}
        \end{pmatrix}U_{\hat{Z_0}}^\top X_t\right\|_F^2\\
        &=\left(1-e^{-\frac{\alpha\eta t}{r}}\right)^2\left\|\hat{X_\infty}V_{\hat{X_\infty}}\begin{pmatrix}
            I_r\\&O_{(a-r)\times(a-r)}
        \end{pmatrix}V_{\hat{X_\infty}}^\top\right\|_F^2\\
        &=\left(1-e^{-\frac{\alpha\eta t}{r}}\right)^2\left(\|\hat{X_\infty}\|_F^2-\sum_{i=r+1}^{\operatorname{hrp\_rank}}\sigma_i(\hat{X_\infty})^2\right).
    \end{align*}
    Thus $\|X_{\infty}-M\|_F^2$ is upper bounded by
    \begin{align*}
        \|X_{\infty}-M\|_F^2\leq \|M\|_F^2-\|X_{\infty}\|_F^2\leq \|M\|_F^2-\|\hat{X_\infty}\|_F^2+\sum_{i=r+1}^{\operatorname{hrp\_rank}}\sigma_i(\hat{X_\infty})^2.
    \end{align*}
    According to Lemma~\ref{sigma-order}, we have
    \begin{align*}
        \sigma_i(\hat{X_\infty}) &\leq \sigma_i(M),
    \end{align*}
    indicating
    \begin{align*}
        \|X_{\infty}-M\|_F^2 &\leq \|M\|_F^2-\|\hat{X_\infty}\|_F^2+\sum_{i=r+1}^{\operatorname{hrp\_rank}}\sigma_i(\hat{X_\infty})^2\\
        &\leq \|M-\hat{X_\infty}\|_F^2+\sum_{i=r+1}^{\operatorname{hrp\_rank}}\sigma_i(M)^2.
    \end{align*}
    By taking expectation to both side, according to Theorem~\ref{ms-loss}, we have
    \begin{align*}
        \mathbb{E}_{U}\mathcal{L}_{\infty}\leq\sum_{i=r+1}^{\operatorname{hrp\_rank}}\sigma_i(M)^2+\frac{a-\operatorname{hrp\_rank}}{2a}\sum_{i=1}^{\max(a,b)}\sigma_i(M)^2.
    \end{align*}
\end{proof}

\section{More Experiment Details}\label{apd-exp-detail}
For the GLUE benchmark using the T5-base model, we use the AdamW optimizer with a linear learning rate scheduler with batch size 32 for all tasks and fine-tuning strategies. For LoRA variants, we inject LoRA to all query and value blocks with $r=4$ and $\alpha=4$. For FPFT, we use a relatively lower learning rate $5\times 10^{-5}$ for all tasks, while for AdaLoRA, we use a relatively higher learning rate $5\times 10^{-3}$. For other LoRA variants, we set the learning rate to the same in each task. Due to the difference across tasks, we set different training epochs and the LoRA learning rate same for all tasks. The detailed information is shown in Table~\ref{apd-nlu-tab}. 

\begin{table*}[htb]
    \centering
    \caption{\label{apd-nlu-tab}Hyperparameter settings for fine-tuning T5-base on GLUE.}
\begin{tabular}{lccccc}
\toprule
 & CoLA & MRPC & QNLI & RTE & STS-B \\
\midrule
Learning rate & $2\times 10^{-3}$ & $2\times 10^{-4}$ & $5\times 10^{-4}$ & $2\times 10^{-3}$ & $2\times 10^{-3}$ \\
Epochs & 10 & 20 & 5 & 20 & 10 \\
\bottomrule
\end{tabular}
\end{table*}

For the math reasoning tasks using a large language model, we use AdamW optimizer with a cosine learning rate scheduler with batch size 16 for all models and fine-tuning strategies. For LoRA variants, we inject LoRA to all linear blocks except the language model head with $r=8$ and $\alpha=8$. For FPFT, we use a relatively lower learning rate $5\times 10^{-5}$ for all models, while for LoRA variants, we use a relatively higher learning rate $4\times 10^{-4}$. During fine-tuning, the base model is loaded with the bfloat16 data type, while LoRA blocks are loaded in the float32 data type. In the inference stage, we set the temperature to 0.8, top-p to 0.95, and let the model output no more than 512 tokens.

\section{Additional Experimental Results}

\subsection{Ablation Studies}\label{ablation}

We conducted ablation studies to evaluate how HRP hyperparameters impact the fine-tuned results. First, we fix $\operatorname{hrp\_step}=100$ and examined the effect of $\operatorname{hrp\_rank}$. As shown in Table~\ref{ablation-rank}, HRP performance is not much sensitive to $\operatorname{hrp\_rank}$. In CoLA, STS-B, the average score, and QNLI for $\operatorname{hrp\_rank}\geq 16$, a higher $\operatorname{hrp\_rank}$ correlates with better performance, indicating a higher $\operatorname{hrp\_rank}$ is required for better initialization. However, in smaller datasets like RTE and MRPC, this trend is not observed, likely because a lower $\operatorname{hrp\_rank}$ is enough for these tasks. 

We also conducted experiments by fixing $\operatorname{hrp\_rank}=128$ to analyze the impact of $\operatorname{hrp\_step}$. As shown in Table~\ref{ablation-step}, where $\operatorname{hrp\_step}=0$ corresponds to random orthogonal initialization, HRP performance is also not sensitive to $\operatorname{hrp\_step}$ when HRP is active. However, whether HRP is active or not has a significant impact. This is as expected, as a higher $\operatorname{hrp\_step}$ is used to mitigate the influence of gradient noise. 

\begin{table*}
    \centering
    \caption{\label{ablation-rank}Ablation studies for $\operatorname{hrp\_rank}$ with $\operatorname{hrp\_step}$ keeping 100. }
\begin{tabular}{lllllll}
\toprule
    $\operatorname{hrp\_rank}$ & CoLA & MRPC & QNLI & RTE & STS-B & Avg. \\\hline
    8 & $54.77_{ \pm 0.04 }$ & $88.81_{ \pm 1.14 }$ & $92.13_{ \pm 0.10 }$ & $73.32_{ \pm 2.66 }$ & $88.35_{ \pm 0.52 }$ & $79.68_{ \pm 0.53 }$ \\
    16 & $55.08_{ \pm 1.68 }$ & $88.73_{ \pm 0.53 }$ & $91.97_{ \pm 0.19 }$ & $74.73_{ \pm 0.88 }$ & $88.68_{ \pm 0.65 }$ & $79.84_{ \pm 0.43 }$ \\
    32 & $55.36_{ \pm 1.82 }$ & $88.56_{ \pm 0.31 }$ & $92.07_{ \pm 0.36 }$ & $74.57_{ \pm 1.45 }$ & $88.84_{ \pm 0.76 }$ & $79.88_{ \pm 0.43 }$ \\
    64 & $56.17_{ \pm 0.43 }$ & $88.48_{ \pm 0.35 }$ & $92.09_{ \pm 0.14 }$ & $74.25_{ \pm 1.45 }$ & $88.87_{ \pm 0.43 }$ & $79.97_{ \pm 0.25 }$ \\
    128 & $56.18_{ \pm 0.09 }$ & {$88.89_{ \pm 0.61 }$} & {$92.23_{ \pm 0.08 }$} & {$74.25_{ \pm 1.90 }$} & {$89.04_{ \pm 0.16 }$} & {$80.12_{ \pm 0.52 }$} \\
\bottomrule
\end{tabular}
\end{table*}

\begin{table*}
    \centering
    \caption{\label{ablation-step}Ablation studies for $\operatorname{hrp\_step}$ with $\operatorname{hrp\_rank}$ keeping 128. }
\begin{tabular}{lllllll}
\toprule
    $\operatorname{hrp\_step}$ & CoLA & MRPC & QNLI & RTE & STS-B & Avg. \\\hline
    0 & $54.84_{ \pm 0.27 }$ & $88.48_{ \pm 0.53 }$ & $91.95_{ \pm 0.16 }$ & $73.89_{ \pm 2.74 }$ & $88.87_{ \pm 0.11 }$ & $79.60_{ \pm 0.57 }$ \\
    50 & $56.32_{ \pm 1.28 }$ & $88.81_{ \pm 0.61 }$ & $92.10_{ \pm 0.28 }$ & $75.09_{ \pm 2.34 }$ & $88.92_{ \pm 0.24 }$ & $80.25_{ \pm 0.63 }$ \\
    100 & $56.18_{ \pm 0.09 }$ & {$88.89_{ \pm 0.61 }$} & {$92.23_{ \pm 0.08 }$} & {$74.25_{ \pm 1.90 }$} & {$89.04_{ \pm 0.16 }$} & {$80.12_{ \pm 0.52 }$} \\
    150 & $56.17_{ \pm 1.38 }$ & $88.64_{ \pm 0.81 }$ & $92.20_{ \pm 0.11 }$ & $76.41_{ \pm 1.51 }$ & $88.98_{ \pm 0.63 }$ & $80.48_{ \pm 0.51 }$ \\
    200 & $57.12_{ \pm 0.65 }$ & $87.83_{ \pm 0.76 }$ & $92.24_{ \pm 0.19 }$ & $74.25_{ \pm 0.95 }$ & $89.22_{ \pm 0.69 }$ & $80.13_{ \pm 0.42 }$ \\
\bottomrule
\end{tabular}
\end{table*}

\subsection{Loss Curves}\label{lc-nlg}
To provide deeper insights into the training dynamics, we present the loss curve for Llama in Figure~\ref{llama-apd}, Qwen in Figure~\ref{qwen-apd}, and Gemma in Figure~\ref{gemma-apd}. The visualization convincingly validates our theoretical analysis, demonstrating that HRP indeed achieves superior converged results initialization. Besides, the loss curves reveal that HRP also accelerates the convergence of LoRA, especially in the beginning stage of fine-tuning. 

\begin{figure}
    \centering
    \includegraphics[height=.3\textheight, bb=0 0 460.8 345.2]{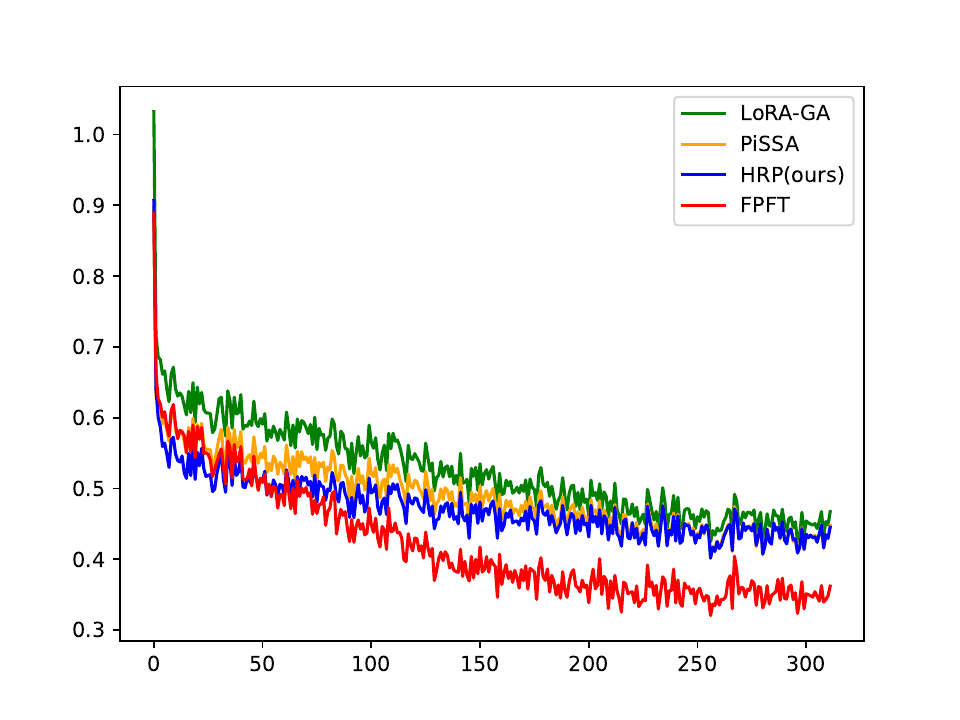}
    \caption{\label{llama-apd}Loss curves for fine-tuning meta-llama/Llama-3.2-1B-Instruct on the MetaMathQA. }
    \includegraphics[height=.3\textheight, bb=0 0 460.8 345.2]{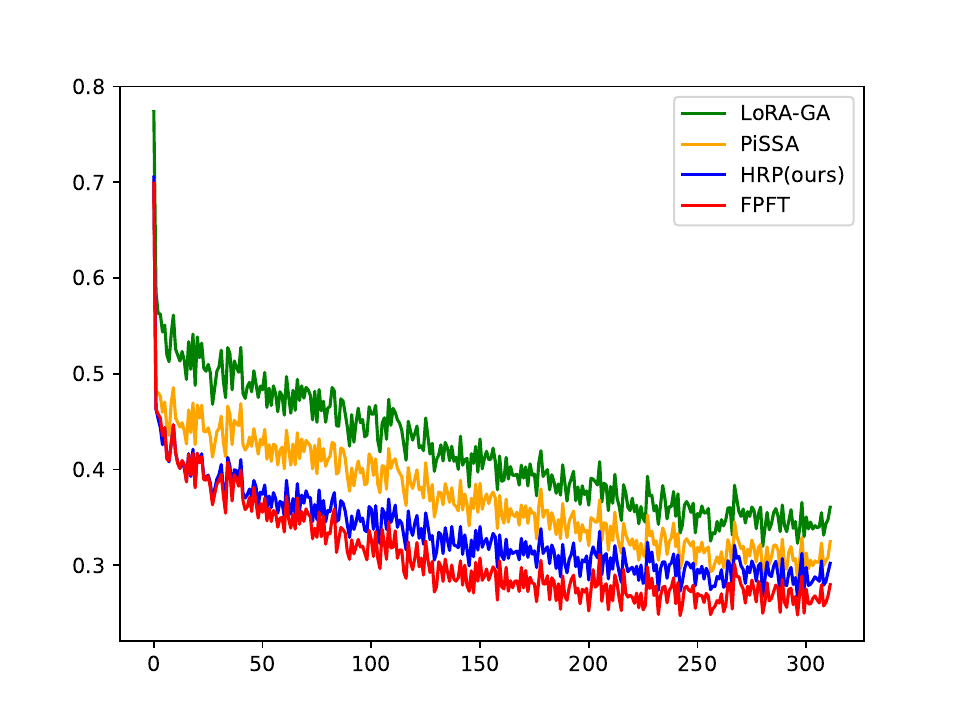}
    \caption{\label{qwen-apd}Loss curves for fine-tuning Qwen/Qwen3-1.7B on the MetaMathQA. }
    \includegraphics[height=.3\textheight, bb=0 0 460.8 345.2]{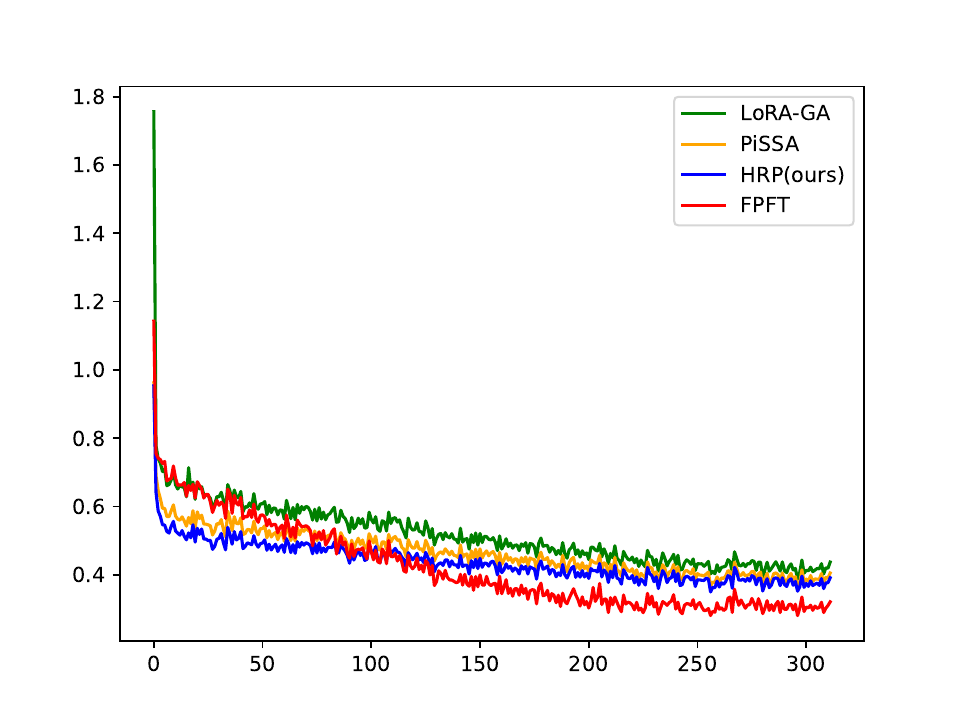}
    \caption{\label{gemma-apd}Loss curves for fine-tuning google/gemma-2-2b-it on the MetaMathQA. }
\end{figure}

% \section{Limitations}\label{check-limiation}

% In this paper, we have demonstrated that HRP makes LoRA achieve better performance than other LoRA variants on the T5-Base (220M), Llama3.2-1B (1.24B), Qwen3-1.7B (2.03B), and Gemma2-2B (2.61B). However, due to computational resource constraints, we have not validated HRP on larger pre-trained models (e.g., Llama 3.3-70B).

% HRP is theoretically designed for better optimization in the training process, which is validated in real-world scenarios through loss curves. However, HRP underperforms some other LoRA variants in some special cases. This limitation suggests a need for considering generalization. 

% Additionally, this paper mainly considers the direction of initialization, while the magnitude should also be important for the convergence rate. We leave it as an interesting future direction. 

% \section{Compute Resources}\label{gpu-info}

% In this paper, all NLU experiments and most NLG experiments were performed on a single NVIDIA H20 and 64-core CPU, except DoRA at Qwen and Gemma on 2 NVIDIA H20 and a 88-core CPU. 

% \section{Broader Impacts}\label{Broader-Impacts}

% This paper presents work whose goal is to advance the initialization of LoRA. A potential impact of this paper is guiding practitioners on more effective initialization for fine-tuning deep learning models using LoRA. After thorough consideration and analysis, it can be firmly stated that this research has no ethical aspects that could raise concerns. 

% \input{sec/checklist}

\end{document}